\documentclass{article}
\usepackage{stmaryrd}

\PassOptionsToPackage{numbers, compress}{natbib}

\usepackage[preprint]{neurips_2024}
\usepackage[utf8]{inputenc} 
\usepackage{wrapfig}
\usepackage[T1]{fontenc}    
\usepackage{hyperref}       
\usepackage{url}            
\usepackage{booktabs}       

\usepackage{algorithm}
\usepackage{algpseudocode}
\usepackage{graphicx}
\usepackage{amsfonts}       
\usepackage{nicefrac}       
\usepackage{microtype}      
\usepackage{dsfont}
\usepackage{xcolor}         
\usepackage{subcaption}
\usepackage{graphicx}
\usepackage{mathrsfs}
\usepackage{amsmath}
\RequirePackage{hypernat}
\usepackage{bbm}
\usepackage{verbatim}
\usepackage{mymathstyle}
\usepackage{cleveref}

\title{LARS-VSA: A Vector Symbolic Architecture For Learning with Abstract Rules}

\author{%
  Mohamed Mejri \\
  School of ECE\\
  Georgia Institute of Technology\\
  Atlanta, USA \\
  \texttt{mmejri3@gatech.edu} \\
  \And
  Chandramouli Amarnath \\
  School of ECE\\
  Georgia Institute of Technology\\
  Atlanta, USA \\
  \texttt{chandamarnath@gatech.edu} \\
  \AND
  Abhijit Chatterjee \\
  School of ECE\\
  Georgia Institute of Technology\\
  Atlanta, USA \\
  \texttt{abhijit.chatterjee@ece.gatech.edu} \\
}

\begin{document}

\maketitle

\begin{abstract}
Human cognition excels at symbolic reasoning, deducing abstract rules from limited samples. This has been explained using symbolic and connectionist approaches, inspiring the development of a neuro-symbolic architecture that combines both paradigms. In parallel, recent studies have proposed the use of a "relational bottleneck" that separates object-level features from abstract rules, allowing learning from limited amounts of data . While powerful, it is vulnerable to the \textit{curse of compositionality} meaning that object representations with similar features tend to interfere with each other.    
In this paper, we leverage hyperdimensional computing, which is inherently robust to such interference to build a compositional architecture. 
We adapt the "relational bottleneck" strategy to a high-dimensional space, incorporating explicit vector binding operations between symbols and relational representations. Additionally, we design a novel high-dimensional attention mechanism that leverages this relational representation. Our system benefits from the low overhead of operations in hyperdimensional space, making it significantly more efficient than the state of the art when evaluated on a variety of test datasets, while maintaining higher or equal accuracy.
\end{abstract}

\section{Introduction}

Analogical reasoning based on relationships between objects is fundamental to human abstraction and creative thinking. This capability differs from our ability to acquire semantic and procedural knowledge from sensory information, processed using contemporary connectionist approaches such as deep neural networks (DNNs). However, most of these techniques fail to extract relational abstract rules from limited samples \cite{rahimi}.
Recent advancements in machine learning have enhanced abstract reasoning capabilities. 
Broadly, these rely on isolating abstract relational rules from the corresponding  representations of objects, such as symbols/(key, query) \cite{abstractor} or (key, values) \cite{esbn}. This approach, called the \textit{relational bottleneck}, is exploited in \cite{abstractor}, using attention mechanisms \cite{vaswani2017attention} to capture relevant correlations between objects in a sequence, producing relational representations. The latter structure is combined with a learnable inductive bias given by \textit{symbols}, linked to the relational representation through a \textit{relational cross-attention} \cite{abstractor} mechanism. More generally, the \textit{relational bottleneck} \cite{webb2024relational}, aims to reduce \textit{catastrophic interference} \cite{tamminen2015specific,bowers2014neural} between object-level and abstract-level features. This catastrophic interference, also known as the \textit{curse of compositionality} is due to the overuse of shared structures and the low-dimensionality of feature representations \cite{webb2024relational}. In\cite{cohen}, it is argued that over compositionality leads to \textit{inefficient generalization} (i.e, requiring expensive processing capabilities). 

The above problem is partially addressed by neuro-symbolic approaches that use quasi-orthogonal high-dimensional vectors \cite{rahimi,menet2024mimonets} for storing relational representations. In \cite{menet2024mimonets}, it is shown that high dimensional vector are less prone to interference. However, recent neuro-symbolic approaches for abstract reasoning \cite{rahimi} rely on explicit binding and unbinding mechanisms, requiring prior knowledge of abstract rules.
This paper presents LARS-VSA (for \textit{L}earning with \textit{A}bstract \textit{R}ule\textit{S}), an approach that combines the advantages of connectionist approaches in capturing implicit abstract rules and the ability of neuro-symbolic architectures to absorb relevant features with low risk of interference.

The key contributions of this paper are:
\begin{itemize}
    \item We are the \textit{first} to propose a strategy for addressing the relational bottleneck problem using a vector symbolic architecture. This performs explicit bindings in high-dimensional space, capturing relationships between symbolic representations of objects distinct from object level features. 
    \item We implement a context-based self-attention mechanism that operates directly in a  bipolar high-dimensional space. Vectors that represent relationships between symbols are developed, decoupling our high-dimensional vector symbolic approach from the need for prior knowledge of abstract rules (seen in prior work \cite{rahimi}).
    \item Our system significantly reduces computational costs by simplifying attention score matrix multiplication to binary operations, offering a lightweight alternative to conventional attention mechanisms.
\end{itemize}

In the following sections, we introduce the intuition behind and explain our novel \textit{self-attention} mechanism for hyperdimensional computing. We then discuss the
proposed vector symbolic architecture for addressing the relational bottleneck problem in high-dimensional spaces.  We compare the performance of LARS-VSA with the Abstractor \cite{abstractor}, a standard transformer architecture \cite{vaswani2017attention}, and other state-of-the-art methods on discriminative relational tasks to demonstrate the  accuracy and cost efficiency of our approach. We evaluate LARS-VSA on various synthetic sequence-to-sequence datasets and complex mathematical problem-solving tasks \cite{saxton2019analysing} to illustrate its potential in real-world applications. Finally, we show the resilience of LARS-VSA to weight heavy quantization \cite{abstractor}.

\section{Related Work}
Several studies \cite{barrett2018measuring,ricci2018same,lake2018generalization} have demonstrated that abstract reasoning and relational representations are notable weaknesses in modern neural networks. Large language models can tackle symbolic reasoning tasks to a certain extent but require extensive training data, which highlights their inability to generalize from limited examples. Consequently, significant effort has recently been directed towards addressing this deficiency. The Relation Network \cite{santoro2017simple} proposed a model for pairwise relations by applying a Multilayer Perceptron (MLP) to concatenated object features. Another architecture, PrediNet \cite{shanahan2020explicitly}, utilizes predicate logic to represent these relations. In \cite{webb2024relational}, the relational bottleneck implementation aims to separate symbolic rules from object level features. Several models are based on the latter idea: \textit{CoRelNet}, introduced in \cite{kerg2022neural} simplifies relational learning by modeling a similarity matrix. A recent study \cite{esbn} introduced an architecture inspired by Neural Turing Machines (NTM) \cite{graves2014neural} that separates relational representations from object feature representations. Building on this concept, the approach of \cite{abstractor} adapted Transformers \cite{vaswani2017attention} for abstract reasoning tasks, by creating an 'abstractor'—a mechanism based on relational and symbolic cross-attention applied to abstract symbols for sequence-to-sequence relational tasks. Additionally, a model known as the Visual Object Centering Relational Abstract architecture (OCRA) \cite{webb2024systematic} maps visual objects to vector embeddings, followed by a transformer network for solving symbolic reasoning problems such as the Raven Progressive Matrices \cite{raven1938raven}. A subsequent study \cite{mondal2024slot} combined and refined OCRA and the Abstractor to address similar challenges.
 
 However, research \cite{rahimi} has shown that hyperdimensional computing (HDC), a neuro-symbolic paradigm \cite{kanerva2009hyperdimensional}, exhibits strong abstraction capabilities through the \textit{blessing of dimensionality} mechanism \cite{menet2024mimonets}. Hyperdimensional Computing (HDC) is recognized for its low computational overhead \cite{mejri2024novel,amrouch2022brain}. In contrast, Transformers \cite{vaswani2017attention} are known to be power-intensive and time-consuming. An explicit binding and unbinding based VSA architecture \cite{rahimi} was developed to solve a visual abstract reasoning problem: Raven Progressive Matrices \cite{raven1938raven}. This \textit{explicitly} represents Raven's Progressive Matrix features (object shapes, color, etc) in the symbolic domain. However, this method requires prior knowledge of relevant object features and the abstract rules governing their relationships and is not always straightforward (e.g, mathematical reasoning tasks).  In relation to prior research and to the best of our knowledge, this paper is the first to propose a hyperdimensional computing based attention mechanism called \textit{LARS-VSA}, that models relational representations, and combines them through an explicit vector-binding operation with high dimensional symbolic vectors to perform sequence-to-sequence relational representation tasks.
 \section{Related Work: Self Attention and Relational Cross Attention}
 \begin{figure}[ht]
    \centering
    \captionsetup[subfigure]{font=bf}
    \begin{subfigure}[b]{.4\textwidth}
        \centering
        \includegraphics[width=\textwidth]{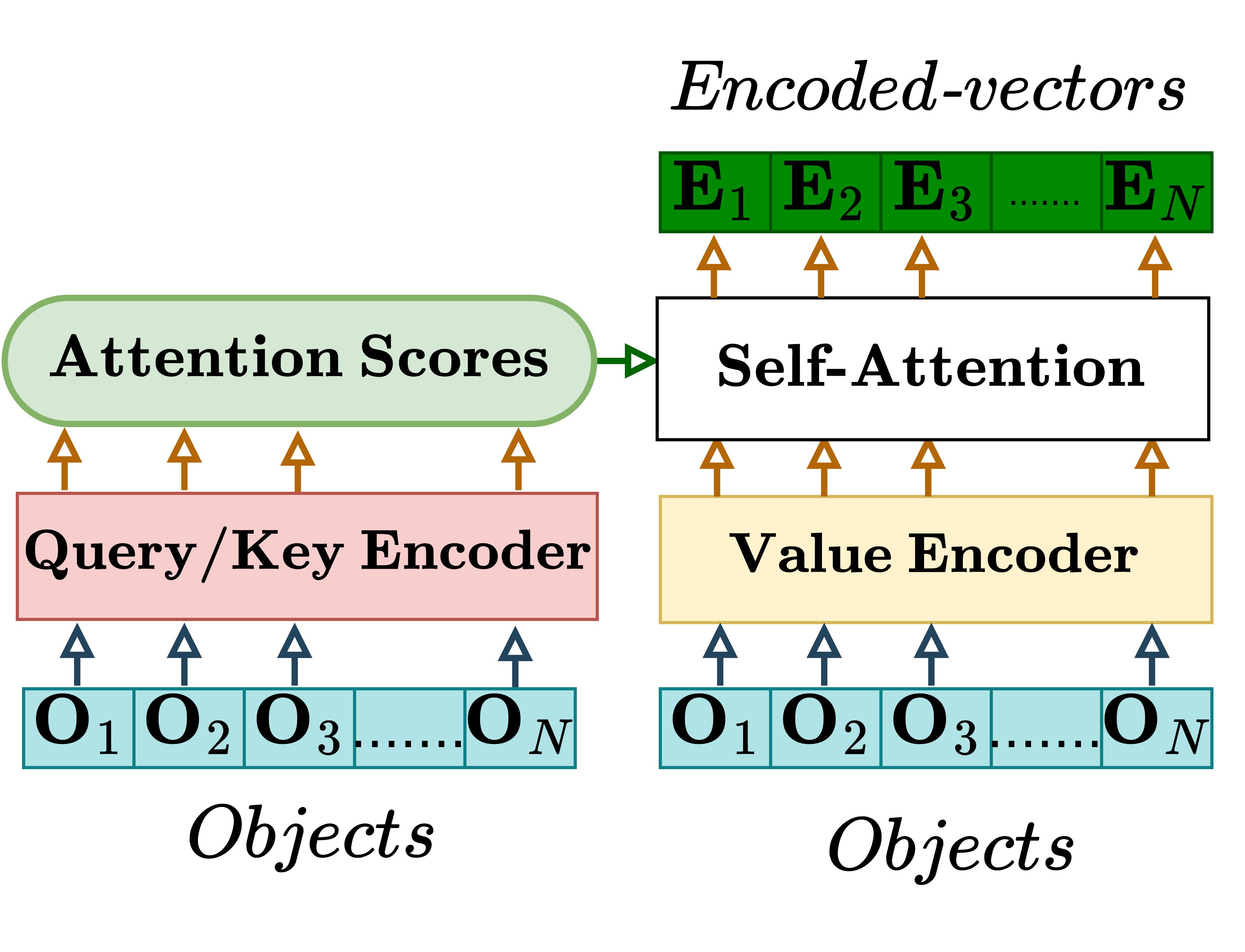}
        \caption{$\mathrm{SelfAttention}(O)$}\label{fig:self_attention}
        \label{fig:transformer}
    \end{subfigure}\hfill
    \begin{subfigure}[b]{.4\textwidth}
        \centering
        \includegraphics[width=\textwidth]{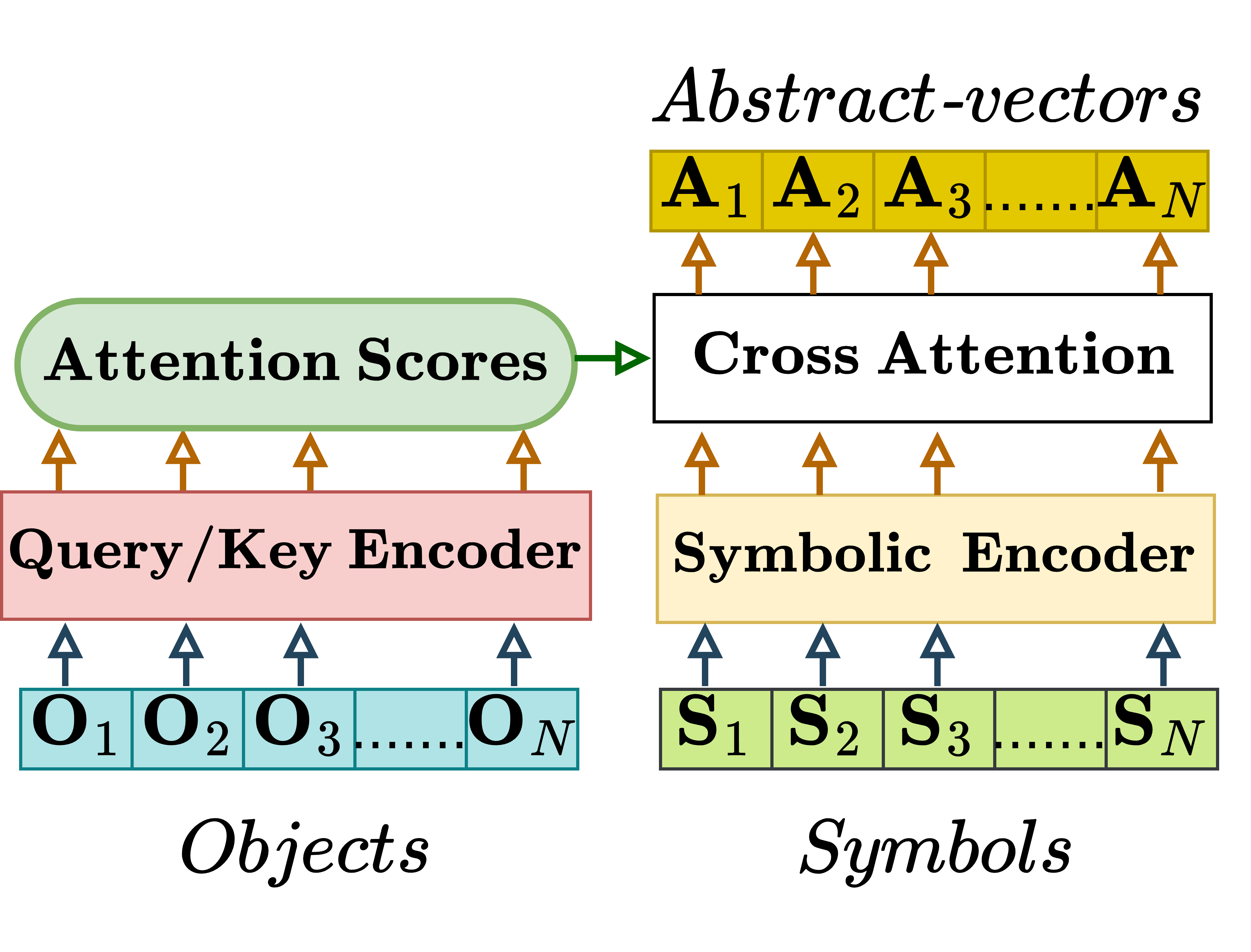}
        \caption{ $\mathrm{RelCrossAttention(O, S)}$}
        \label{fig:abstractor}
    \end{subfigure}\hfill
    \caption{Relational cross-attention and self-attention. \textit{We show a single head of multi-attention}. In blue, the \textit{objects} related by $r(.,.)$; in green, the inductive bias denoted by \textit{symbols}}
    \label{fig:overview}
\end{figure}
Figure \ref{fig:overview} shows a single head for different attention mechanisms from prior work: Figure \ref{fig:transformer} refers to the regular \textit{transformer self-attention mechanism} applied to a sequence of objects $O_{1..N}$ as in \cite{vaswani2017attention}. The Query/Key Encoder in Figure. \ref{fig:transformer} uses learnable projections $\phi_Q : O \mapsto O\cdot W_Q $ and $\phi_K : O \mapsto  O\cdot W_K$. A self attention score matrix is computed from these projections. The latter matrix is used as weights for the value vectors derived from the objects $O$ using a learnable projection function $\phi_V : O \mapsto O\cdot W_V $ to generate encoded vectors $(E_{1},\ldots,E_{N})$. 
Equation (3) summarizes the \textit{multi-head self-attention} mechanism applied to a sequence of objects $O_{1..N} = \left[ O_1,\ldots O_N \right]$ across $H$ heads. Here, $W_o$ constitutes a weighted summation of the outputs of H different self-attention heads. 
$Self Attention(O) = \mathrm{concat}(\hat{E}^{1}, \ldots, \hat{E}^{H}) \cdot W_o, \text{where} \; \; \hat{E}^i = [E_1^i, E_2^i, \ldots, E_N^i] = \mathrm{Softmax}((O \cdot W_q^i) (O \cdot W_k^i)^\top) (O \cdot W_v^i), 1 \leq i \leq H  (3)$   


The work of \cite{esbn} was the first to propose the use of segregated datapaths for processing keys and values with the goal of processing relational information distinctly from object level features. This led to the concept of the Abstractor \cite{abstractor} in which the self-attention mechanism of a transformer was adapted to allow a relational attention mechanism to be implemented and integrated with symbolic processing.
Figure \ref{fig:abstractor} shows a $\mathrm{RelCrossAttention}$ mechanism \cite{abstractor}  (for a single head) that replaces object-level features (Values $V$) with abstract symbols denoted as $S \in \mathbb{R}^{N \times F}$.The symbols are projected into value vectors through a projection function: $\phi_O: S\mapsto S\cdot W_O $
We have $Q \leftarrow O$, $K \leftarrow O$, $V \leftarrow S$. In this case, $\sigma_{\mathrm{rel}}$ refers to the Softmax function. The formalism for multi head attention is $mathrm{RelCrossAttention}(O, S) = \mathrm{concat}(A^1, \ldots, A^H) W_o, \text{ where } A^i = \sigma_{\mathrm{rel}}((O \cdot W_q^i) (O \cdot W_k^i)^\top) S \cdot W_O^i.)
$
 \section{LARS-VSA: Attention Mechanism With Hypervector Bundling}


Our hyperdimensional symbolic attention mechanism aims to extract abstract relationships $r(.,.) \in \mathbb{R}^{H}$ between objects $O_1$,...,$O_n$. These abstract relations are modeled by a pairwise relationship (i.e \textit{inner product}) between encoded objects (using learnable query and key encodings) $\{r(O_i,O_j)\}_{i,j \in \llbracket 0,N \rrbracket}$ \cite{abstractor}. 
However, the query and key embedding in prior work are represented in a low dimensional space (i.e, dimension less than hundred). In this research, the objects are projected onto a \textit{high D-dimensional hyperspace} (i.e $D \geq 1000$) with data represented as \textit{hypervectors}. In \cite{menet2024mimonets}, it is shown that as the dimensionality of hypervectors increases, they become \textit{orthogonal}, making a relational modeling approach such as \cite{abstractor} difficult to apply in the hyperdimensional space. Applying inner product to orthogonal vectors leads to very low attention scores that do not reflect the actual object correlations and relationships. One way to address this issue is to model the relationship between objects \textit{indirectly} by measuring their correlation to a \textit{local context}. To illustrate this, consider two objects, \(A\) and \(B\). Instead of examining the explicit correlation between objects \(A\) and \(B\),
we assess the correlation/relationship between the object \(A\) and a local context formed by the sequence "\(AB\)". 

The representation of a sequence of objects in a hyperdimensional space is well defined using the \textit{bundling} operation (i.e., \(\oplus\)) between sequence elements. Given two objects \(O_i\) and \(O_j\), we project them onto a hyperspace using a set of $N$ learnable projection functions: $\{{{\phi_{B}}_i}\}_{i=0}^{i=N}: \mathbb{R}^{F} \rightarrow \mathbb{R}^D$, for a system with $N$ objects. Hence the relation $r(O_i,O_j)$ can expressed as in Equation \ref{hdcor}.
 \begin{equation}
     r(O_i,O_j) = \left(f({\phi_{B}}_i(O_i),{\phi_{B}}_j(O_j)\right) \in \mathbb{R}
     \label{hdcor}
 \end{equation}

The function $f: \mathbb{R}^{D}\times \mathbb{R}^{D} \rightarrow \mathbb{R}$ implements the \textit{indirect} correlation we discussed earlier. for ${h_{O}}_1,{h_{O}}_2 \in \mathbb{R}^D$ we have 
\begin{equation}
    f({h_{O}}_1,{h_{O}}_2) =  cos({h_{O}}_1,{h_{O}}_1 \oplus {h_{O}}_2) = \frac{\langle {h_{O}}_1,{h_{O}}_1 \oplus {h_{O}}_2\rangle}{D}
\end{equation}
here we define \textit{bundling} (i.e,  $\oplus$) as $h_1 \oplus h_2 = \mathrm{sign}(h_1+h_2)$. $\operatorname{sign}(x) = -\mathds{1}_{\{x < 0\}} + \mathds{1}_{\{x > 0\}}$ replaces the binary coordinate-wise majority in the bipolar domain\cite{sign}. Considering all pairs of objects, we construct a \textit{relational matrix} consisting of $R$ = $[ r(O_i,O_j)] \in \mathbb{R}^{N \times N}$. The bundling operation implicitly acts as a majority vote between two object elements. It captures the \textit{dominant/relevant} features of an object sequence rather than just common features. Indeed, the sign of each hypervector element follows the sign of the element with higher magnitude which is amplified by dominant features of object sequences during training.
Figure\ref{fig:local_hd_attention} illustrates the pipeline for HDSymbolicAttention from end to end applied to objects $O_{1..N}$. The objects were first mapped to hypervectors ${h_{O}}_{1..N}$ through  learnable projection function $\phi_{H}: O \mapsto O\cdot W_{H}$. We extract the attention scores $r_{ij} = f({h_{O}}_i,{h_{O}}_j) =  cos \left({h_{O}}_{i},{h_{O}}_i\oplus {h_{O}}_j \right)$ from the hypervectors ${h_{O}}_{i}$ and ${h_{O}}_j$. This means that the hypervector ${h_{O}}_i$ attention score is computed with respect to its context illustrated in ${h_{O}}_i\oplus {h_{O}}_j$. these attention scores are then forwarded to a softmax function and we used to compute a weighted sum of ${h_{O}}_{i}$. We should notice that each hypervector ${h_{O}}_{i}$ mainly contributes to its corresponding output hypervector since $r_{ii} \leq r_{ij} \forall i \neq j$. We should also notice that although the cosine similarity is theoretically between -1 and +1, since we are working in hyperdimensional space, the cosine similarity is less likely to be negative\cite{kim2021cascadehd}.       


 \begin{wrapfigure}{L}{0.45\textwidth}
	\vskip-5pt
	\begin{tabular}{c}
	\includegraphics[width=0.45\textwidth]{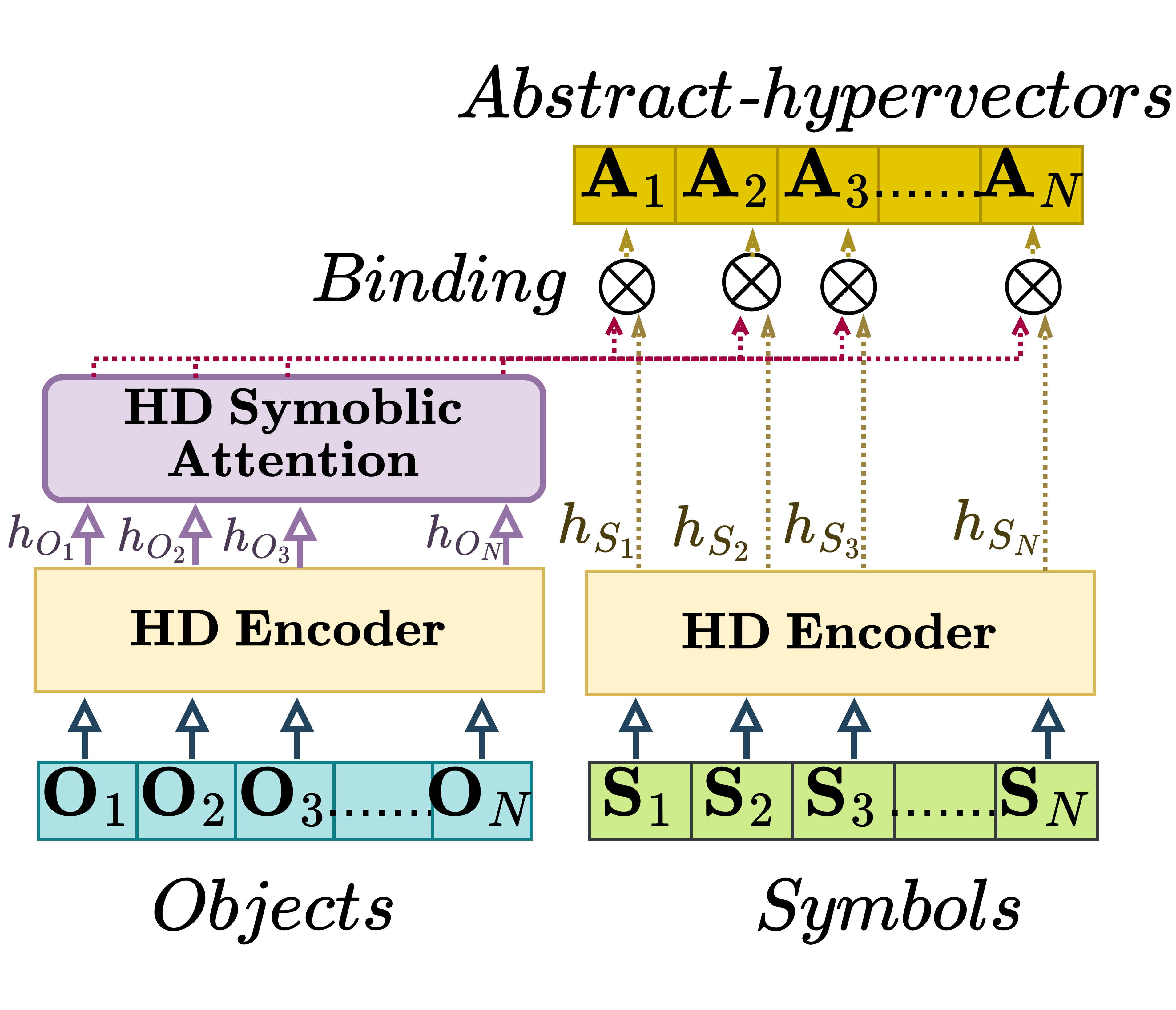}\\[-5pt]
	\end{tabular}
	\caption{  $\mathrm{HDSymbolicAttention(O)}$ $\otimes$ $h_S$}\label{fig:hd}
 \label{hdatten}
 \vspace{-0.1in}
\end{wrapfigure}

\section{LARS-VSA: Symbolic Reasoning with Hypervectors}
\label{lars}

Our proposed system shown in Figure \ref{fig:hd} aims to resolve the relational bottleneck problem in a high dimensional space using an explicit \textit{binding} mechanism. Similar to \cite{abstractor}, the \textit{object representations are separated from the trainable symbol representations} to minimize interference and construct abstract relationships between objects. A set of objects $O_{1..N} = \left[O_{1},\ldots O_{N}\right]$ is forwarded to an \textit{HD Encoder} function ($\phi_{B}:  O\mapsto O\cdot W_{B}$) where $W_{B} \in \{-1,+1\}^{F\times D}$ is a learnable bipolar projection matrix  projection that generates a set of high dimensional vectors $h_{O_{1..N}} = \left[ h_{O_1},\ldots h_{O_N} \right]$. The generated set of high dimensional vectors  is then forwarded to an \textit{HD Symbolic attention mechanism} that produces a new set of encoded high dimensional vectors. These vectors are then \textit{bound} with the symbolic high dimensional vectors $h_{S_{1..N}}= S_{1..N}\cdot W_{B}$. The final output is a set of abstract high dimensional vectors $A_{i..N} \in \mathbb{R}^{N\times D}$ that feeds into other layers of the network. 

Figure \ref{fig:local_hd_attention} illustrates how the $\mathit{HDSymbolicAttention}$ is applied to the set of objects $O_{1..N}=\left[ O_{1},\ldots O_{2}\right]$ of Figure \ref{hdatten}. Each object $O_{i}, i\in \llbracket1..N\rrbracket$ is mapped to a high dimensional vector ${h_{O}}_{i}$ called \textit{object hypervector} through a learnable projection function $\phi_{B}$. All the objects and symbols are mapped to hypervectors using the \textit{same} projection function. This set of object hypervectors is used to compute the attention scores matrix (before softmax normalization) $\left[ R_{1},\ldots R_{N}\right] \in \mathbb{R}^{N\times N}$. Each set of attention scores $R_{i}, i\in \llbracket1..N\rrbracket$ is forwarded to a softmax function $\sigma$. Each encoded object hypervector (i.e, the output of the HDSymbolicAttention module at a node $i$) is a weighted summation of all the \textit{object hypervectors}. The weights above, are the set of attention scores $\sigma(R_{i})$. Each element $r_{ij}$ of the attention scores $R_{i} = \left[ r_{i1},..,r_{iN}\right] \in \mathbb{R}^{N}$ represents the correlation (i.e, similarity or angle in the context of hyperdimensional computing) between the object hypervector ${h_{O}}_{i}$ and the \textit{context hypervector} defined as a ${h_{O}}_{i}$ \textit{bundled} ${h_{O}}_{j}$ together ( ${h_{O}}_{i}\oplus{h_{O}}_{j}$) (i.e, $r_{ij} =  cos\left({h_{O}}_{i},{h_{O}}_{i}\oplus{h_{O}}_{j}\right)$)  
For a single head, the mechanism illustrated in Figure\ref{fig:hd} can be written as:
\begin{equation}
    \begin{split}
    \mathrm{HDSymbolicAttention(O_{1..N})\otimes h_{s}} &= \left[\sigma(\left[R_1,...R_N\right])\cdot \left( O_{1..N}.W_{B}\right)^{T})\right] \otimes \left(S_{1..N}.W_{B}\right)  
    \end{split}
\end{equation}
We note that the HDSymbolicAttention mechanism illustrated in   Figure\ref{fig:local_hd_attention} and Figure\ref{fig:hd} refers to a \textit{single} head of a multi head attention. A multi head attention version was used in the prototype implemented in this research.  
 \begin{figure}[ht]
 \vskip-5pt
\begin{center}

    \includegraphics[width=0.8\linewidth]{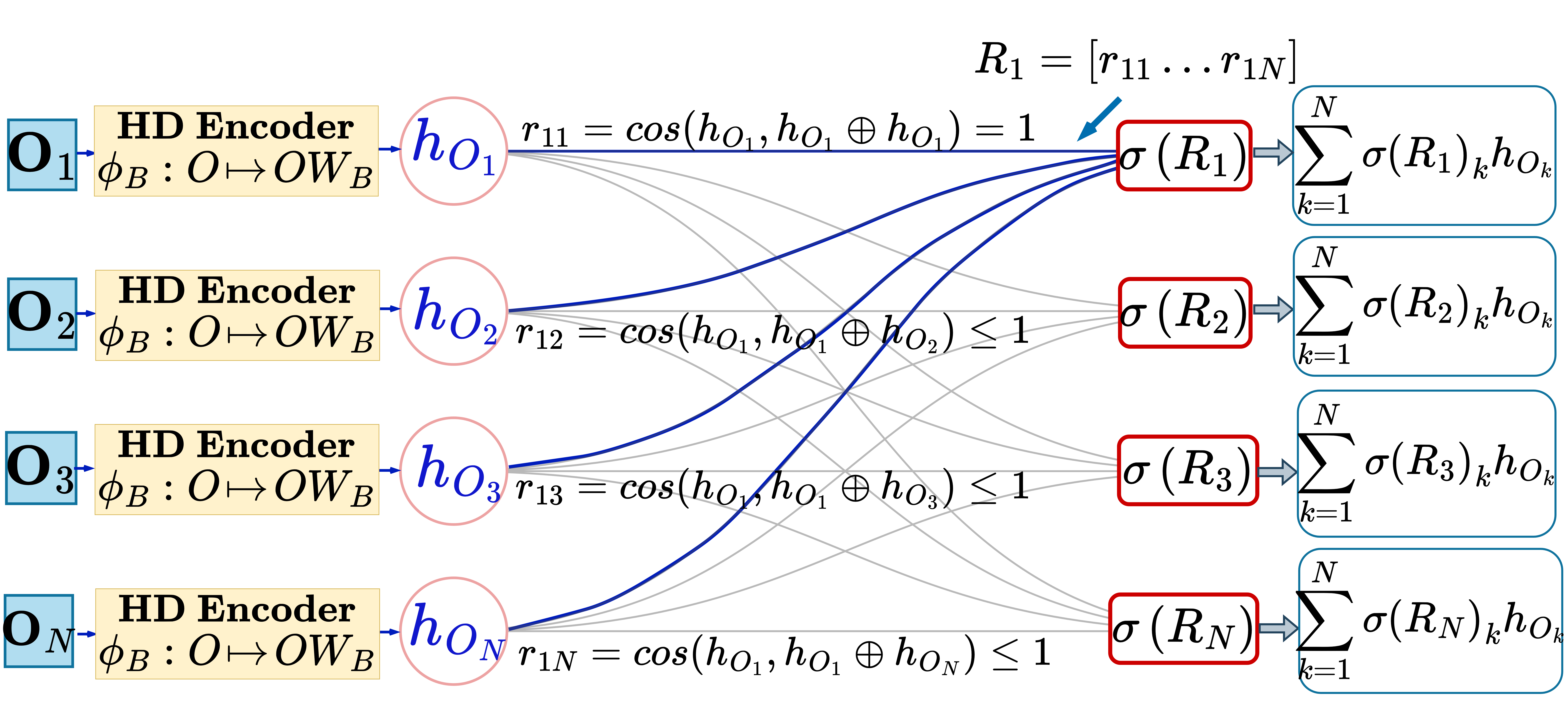}
    \caption{$\mathrm{HDSymbolicAttention}(O_{1..N})$}
    \label{fig:local_hd_attention}
    
\end{center}
\vspace{-0.3in}
\end{figure}

\subsection{Binarized HD attention score}
Calculating attention scores is the most energy-intensive operation in self-attention \cite{ham20203}. We tackle this in our HDSymbolicAttention mechanism by taking advantage of the fact that binarization of the dot product in attention score calculation allows up to\textit{2x memory savings and $60\%$ more speedup} with low accuracy loss \cite{xnornet} over real number multiply-accumulate computations used in transformer arithmetic.

Figure\ref{fig:Bhd1} illustrates an example of our proposed binarized HDSymbolicAttention mechanism between two object hypervectors $h_{{O}_1}$ and $h_{{O}_2}$ involving well defined operations (i.e, binary AND, $L_{0}$ norm, element-wise real addition). The attention mechanism is expressed using the cosine similarity function $cos$ as follows: \ref{hdlocatt}
$
    f(h_{{O}_1},h_{{O}_2}) = cos(h_{{O}_1},h_{{O}_1}\oplus h_{{O}_2})
    \label{hdlocatt}$

 \begin{wrapfigure}{R}{0.4\textwidth}
	\vskip-5pt
	\begin{tabular}{c}	\includegraphics[width=0.4\textwidth]{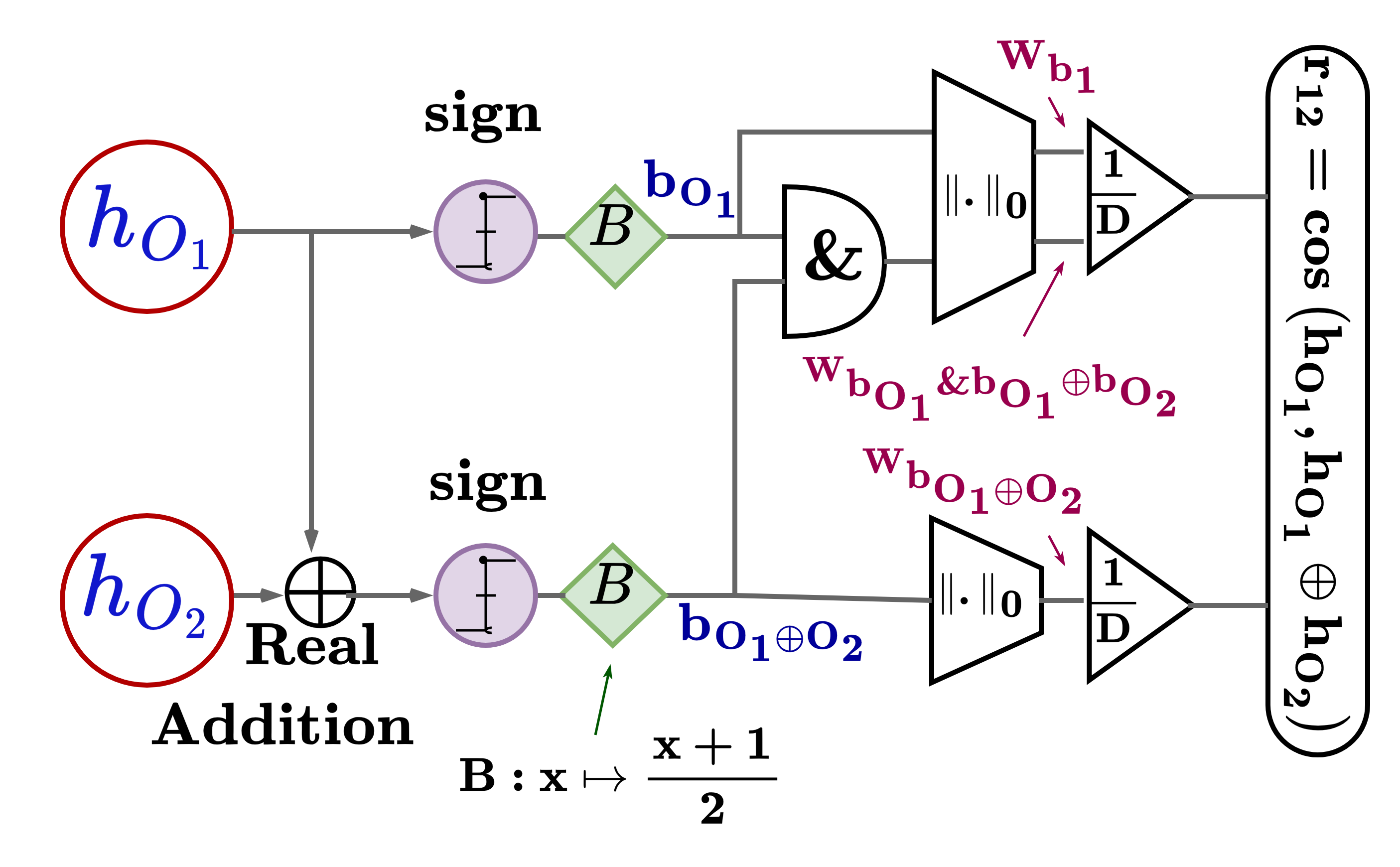}\\[-5pt]
	\end{tabular}
	\caption{Binarized Attention Score Function $r_{12} = cos(h_{{O}_1},h_{{O}_1}\oplus h_{{O}_2})$}\label{fig:Bhd1}
\end{wrapfigure}

Let $b_{{O}_{1}}$ be the binary version of the bipolar hypervector derived from $h_{{O}_{1}}$ by applying a $sign$ function and a function $B(.)$ (i.e., a shift and scale). The binary vector $b_{O_{1}\oplus{O_{2}}}$ is obtained from the same process as $b_{{O}_{1}}$ but applied to $h_{{O}_{1}}\oplus{h_{{O}_{2}}}$. Both $b_{{O}_{1}}$ and $b_{O_{1}\oplus{O_{2}}}$ are forwarded to an element-wise AND gate, and we derive the $L_{0}$ norm of the generated vector to get $\mathrm{w}_{{{b_{O}}_{1}}\&{b_{O_{1}\oplus{O_{2}}}}}$. Similarly, we obtain $\mathrm{w}_{{b_{O}}_{1}}$ and $\mathrm{w}_{{b_{O_{1}\oplus{O_{2}}}}}$. Lemma \ref{equi_cos} and the notation in Figure \ref{fig:Bhd1} allow us to simplify the function $r_{12} = f({h_{O}}_{1},h_{{O}_{2}})$ using binary operations according to Equation \ref{eq:binary}
\begin{equation}
     f(h_1,h_2) =  \frac{1}{D}\left( 4\mathrm{w}_{{{b_{O}}_{1}}\&{b_{O_{1}\oplus{O_{2}}}}} - 2\mathrm{w}_{{b_{O}}_{1}} - 2\mathrm{w}_{{b_{O}}_{1}}\right) +1 
     \label{eq:binary}
\end{equation}

\begin{lemma}\label{equi_cos}
Given two D-dimensional vectors $h_1$ and $h_2$  that are \textit{bipolar} meaning ${h_{1/2}}_i \in \{ -1,+1\} \; \; \forall i \in \llbracket 1, D \rrbracket$ and let $B x: \mapsto \frac{x+1}{2}$ maps bipolar words to binary domain, $\mathrm{\&}$ is the binary 
AND and $\norm{.}_0$ is the zero-Norm. then cosine similarity between $h_1$ and $h_2$ can be expressed as following \ref{bin_ham}  
\begin{equation}
    \mathrm{cos}(h_1,h_2) = 1 + \frac{1}{D}\left[ 4 {\left(B(h_1) \mathrm{\&} B(h_2)\right)}_0 -2{ \left(B(h_1)\right)}_0 - 2 { \left( B(h_2)\right) }_0 \right]
    \label{bin_ham}
\end{equation} 
\end{lemma}
\begin{proof}
Let $h_1$ and $h_2$ two D-dimensional \textit{bipolar} vectors then their cosine similarity could be expressed as $cos(h_1,h_2) = \frac{\langle h_1,h_2\rangle}{\norm{h_1}_2 \cdot \norm{h_2}_2} \label{cos_sim}$
we have $\norm{h_1}^2_2 = \norm{h_2}^2_2 = \sum_{i=1}^{D} \left( \pm 1\right)^2 = D$.
we also have $h_2 = 2B(h_2)-\mathbbm{1}$ and $h_1 = 2B(h_1)-\mathbbm{1}$, where $\mathbbm{1}$ means a D-dimensional vector with all ones. Hence, \ref{cos_sim} becomes
\begin{equation}
\begin{aligned}
    \cos(h_1,h_2) & = \frac{\langle 2B(h_1)-\mathbbm{1},2B(h_2)-\mathbbm{1} \rangle}{D} 
    = 1+ \frac{4 \norm{B(h_1)\mathrm{\&} B(h_2)}_0- 2\norm{B(h_1)}_0 - 2\norm{B(h_2)}_0}{D} 
\end{aligned}
\end{equation}
\end{proof}
\vspace{-0.4in}
\section{LARS-VSA Architecture}
The LARS-VSA architecture is composed of $H$  HDSymbolicAttention heads (line 2). In (line 3) we apply the HDSymbolicAttention module discussed earlier to the set of objects ${O}_{1..N} = \left( O_1,O_2,...,O_N\right)$ and used the symbol hypervectors $h_{S}^{h-1}$ to get the abstract rules in high dimensional space, $\mathrm{A}_{1..N}^{h-1} = \left( A_{1}, ..., A_{N}\right)$. In line 4, the abstract hypervectors $\mathrm{A}_{1..N}^{h-1}$ are scaled using a BatchNormalization layer. Finally, all the abstract rules vector from different heads are \textit{bundled} together in the real domain through a summation operation. In the line 7, the high dimensional abstract vectors $\mathrm{A}_{1..N}^{H}$ are then compressed using a \textit{Global Average Pooling} layer that reduces the dimensionality of the hypervectors to the original object's dimension. The \textit{holistic} \cite{kanerva2009hyperdimensional} property of hyperdimensional computing ensures none or low information loss after dimensionality reduction for hypervectors, avoiding catastrophic interference.

\begin{algorithm}[ht]
  \caption{LARS-VSA Module}\label{alg:nvsa_architecture}
  \begin{algorithmic}[1]
      \State $\mathrm{A}_{1..N}^{0} \gets \emptyset$
      
      \For {$h \gets 1$ to $H$}
      
      \State $\mathrm{A}_{1..N}^{h-1} \gets \textit{HDSymbolicAttention}(\mathrm{O}_{1..N}) \otimes (h_{S}^{h-1})$
      
      \State $\mathrm{A}_{1..N}^{h-1} \gets \textit{BatchNorm}(\mathrm{A}_{1..N}^{h-1})$
      
      \State $\mathrm{A}_{1..N}^{h} \gets \mathrm{A}_{1..N}^{h} + \mathrm{A}_{1..N}^{h-1}$
      
      \EndFor
      \State $\mathrm{A}_{1..N}^{H} \gets \textit{GlobalAvgPooling}(\mathrm{A}_{1..N}^{H})$      
      \State \textbf{return} $\mathrm{A}_{1..N}^{H}$
  \end{algorithmic}
\end{algorithm}

 \begin{wrapfigure}{R}{0.4\textwidth}
	\begin{tabular}{c}	\includegraphics[width=0.4\textwidth]{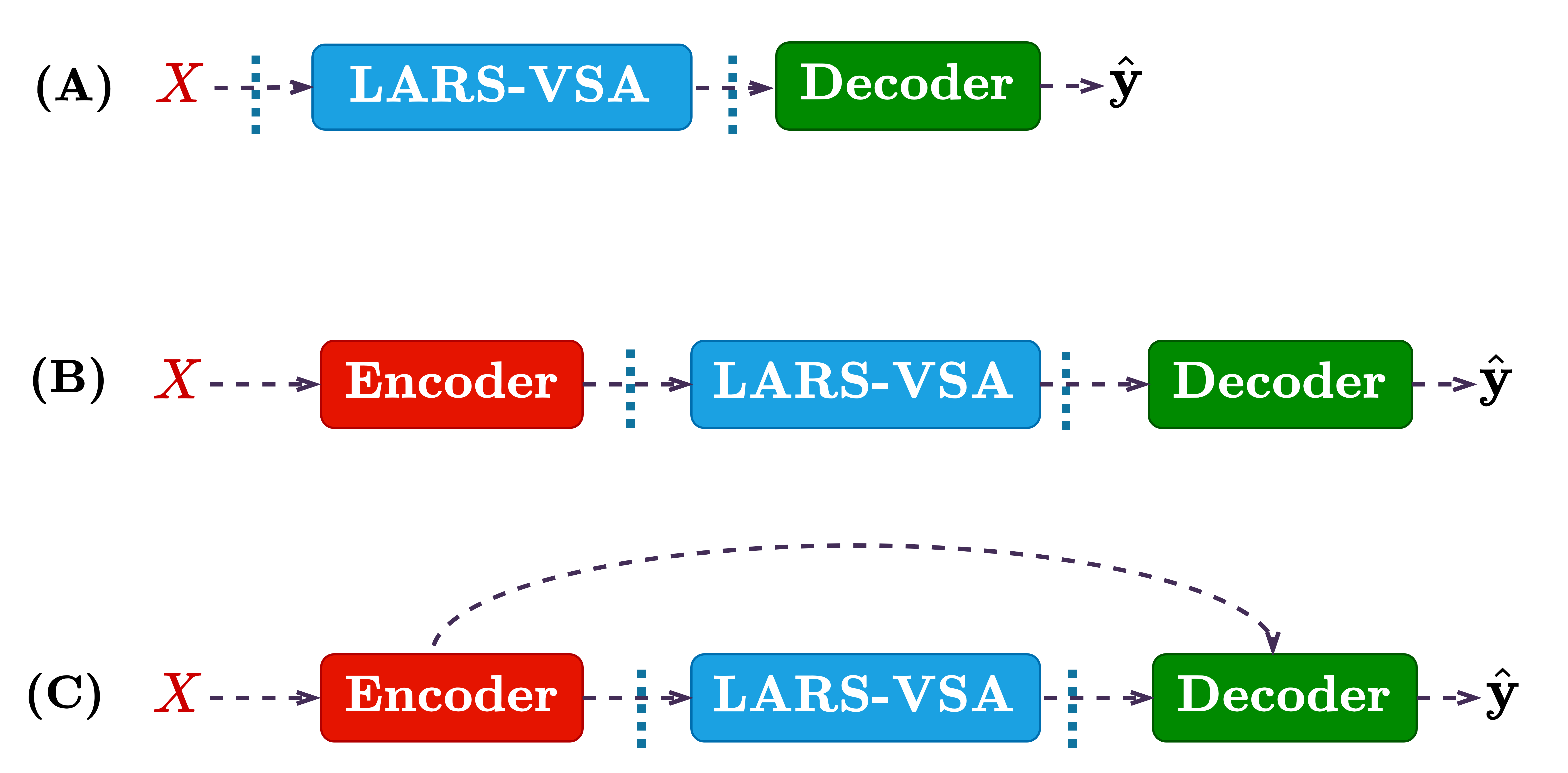}\\[-5pt]
	\end{tabular}
    \caption{Examples of evaluated LARS-VSA based pipelines}\label{fig:architecture}
\end{wrapfigure}

Transformers perform object-level information retrieval and projection onto future outcomes from large volumes of training data, particularly in natural language processing, relying on decoding of relational representations. In \cite{abstractor}, self-attention and cross-attention are utilized to construct a decoder capable of producing NLP-like solution instances. We integrated the decoder structure from \cite{abstractor} into our LARS-VSA pipeline. This excels in both partial as well as purely abstract reasoning tasks, necessitating a specific type of information generation, such as math problem-solving or object sorting. For partially abstract mathematical reasoning tasks, a skip connection between the encoder and decoder is required due to the identical structure of the input and output (both are text), as depicted in structure (C) of Figure \ref{fig:architecture}. However, for purely abstract sorting problems, this skip connection is unnecessary, as illustrated in structure (B) of Figure \ref{fig:architecture}.

For classification-based abstract reasoning tasks, like SET \cite{abstractor} or pairwise ordering \cite{abstractor}, a simple decoder composed of linear layers suffices, as shown in structure (A) of Figure \ref{fig:architecture}. The approach of \cite{abstractor} requires a transformer-based encoder to extract object-level correlations from initial inputs before forwarding them to abstractor modules, despite increasing the computational expenses of the already heavy abstractor pipeline. In our study, we employed a basic BatchNormalization layer followed by dropout to prevent overfitting. The Encoder here functions as a trainable standard scalar preprocessing step. Consequently, object-level correlation and abstraction are primarily provided by the LARS-VSA module.
For sequence to sequence relational reasoning task, we used a cross-attention mechanism derived from \cite{abstractor} between encoded ground truth target and the generated abstract states from the LARS-VSA model.

\section{Experiments}
\subsection{Discriminative Relational Tasks}

\begin{figure}[ht]
    \vskip-.2in
    \centering
    \includegraphics[width=0.8\textwidth]{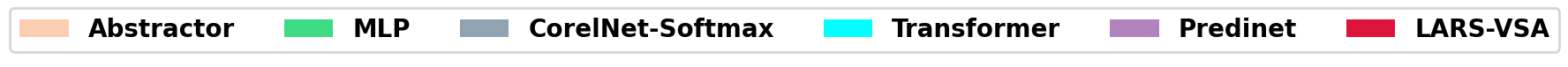}
    \begin{subfigure}[t]{0.3\textwidth}
        \centering\captionsetup{width=.9\linewidth}
        \includegraphics[width=\textwidth]{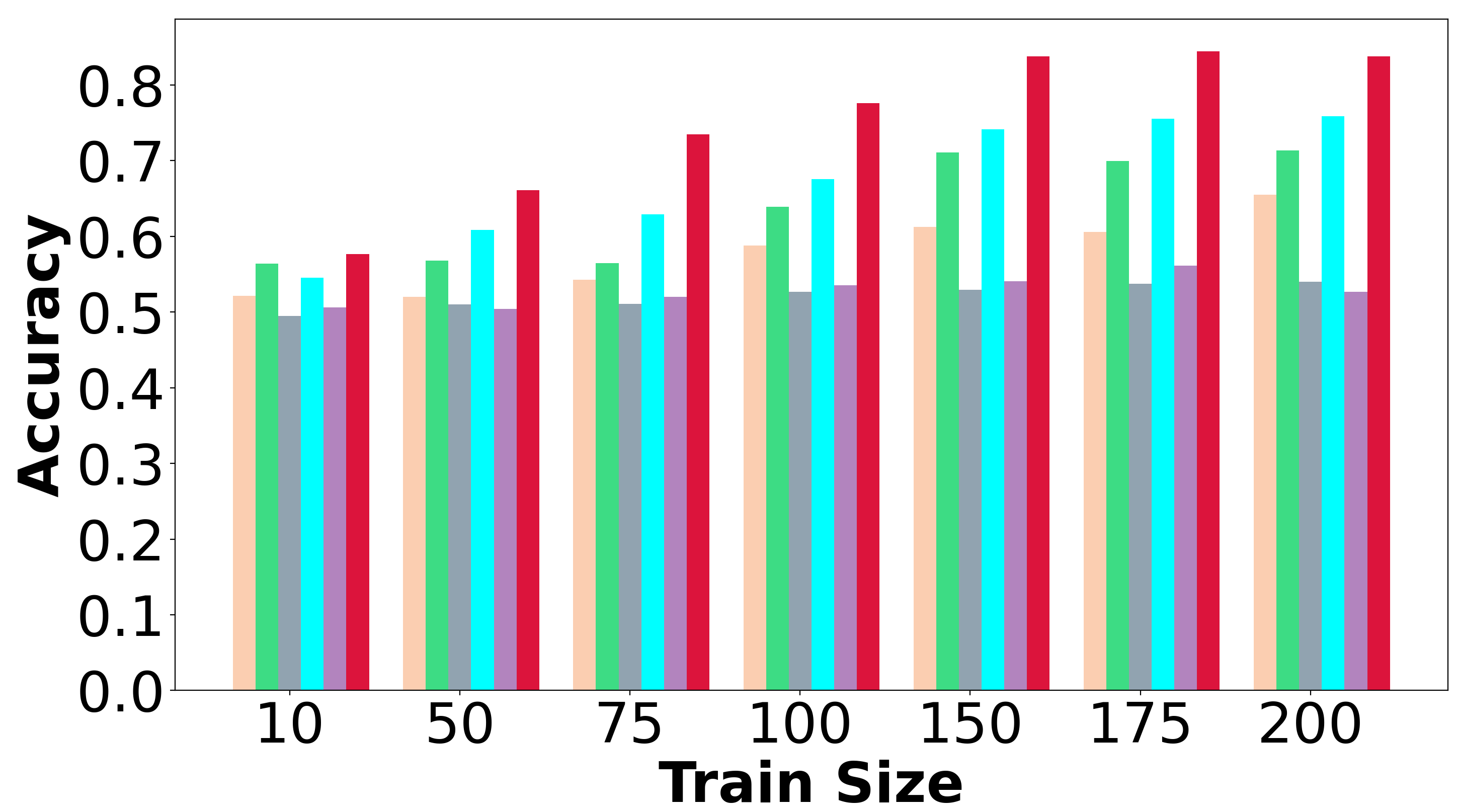}
        \vskip-5pt
        \caption{Learning $\prec$ relation  accuracy}\label{fig:exp_order_relation}
    \end{subfigure}
    \begin{subfigure}[t]{0.32\textwidth}
        \centering\captionsetup{width=.9\linewidth}
        \includegraphics[width=\textwidth]{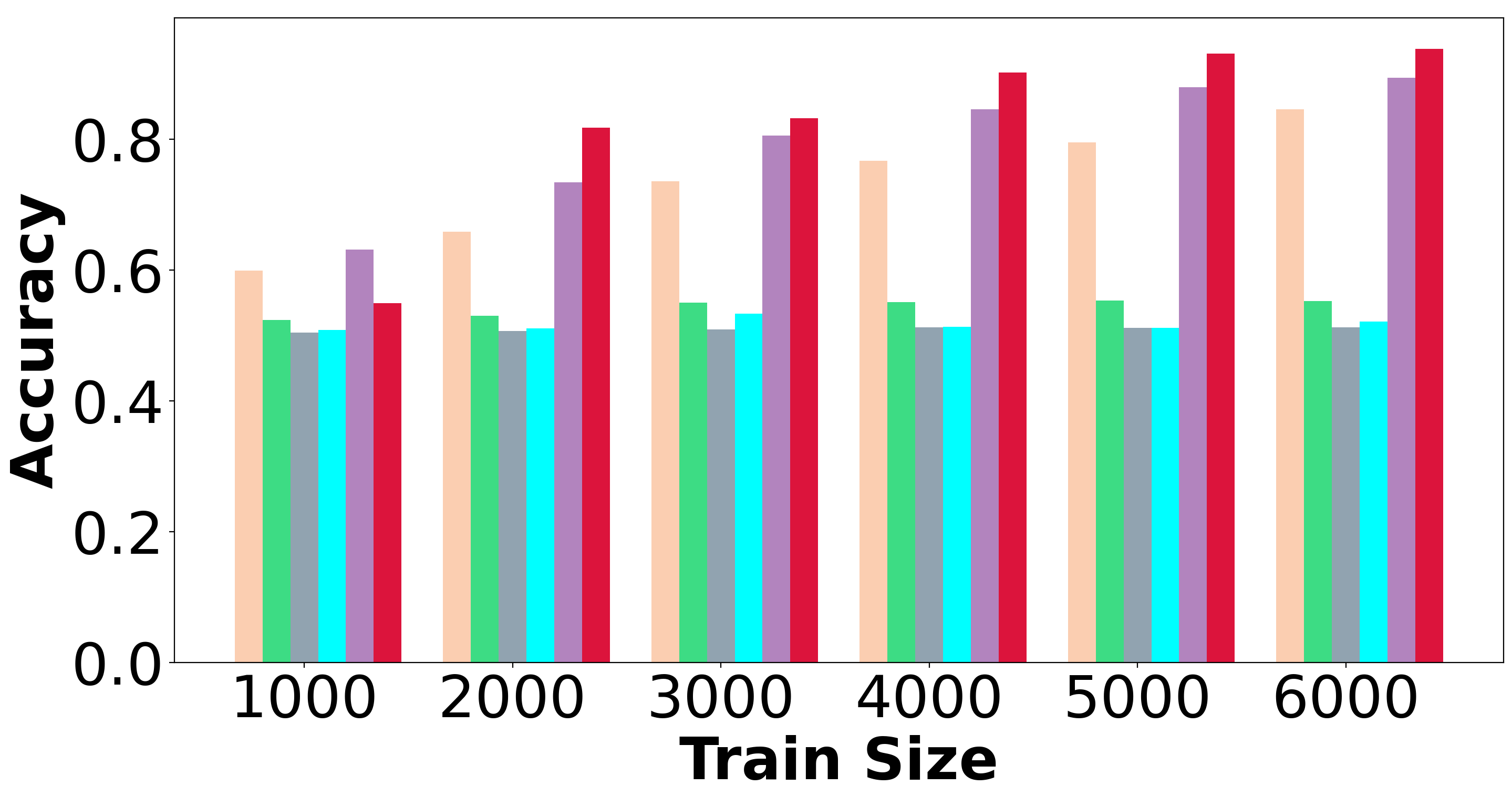}
        \vskip-5pt
        \caption{SET Classification accuracy}\label{fig:exp_set_classification}
    \end{subfigure}
    \begin{subfigure}[t]{0.32\textwidth}
        \centering\captionsetup{width=.9\linewidth}
        \includegraphics[width=\textwidth]{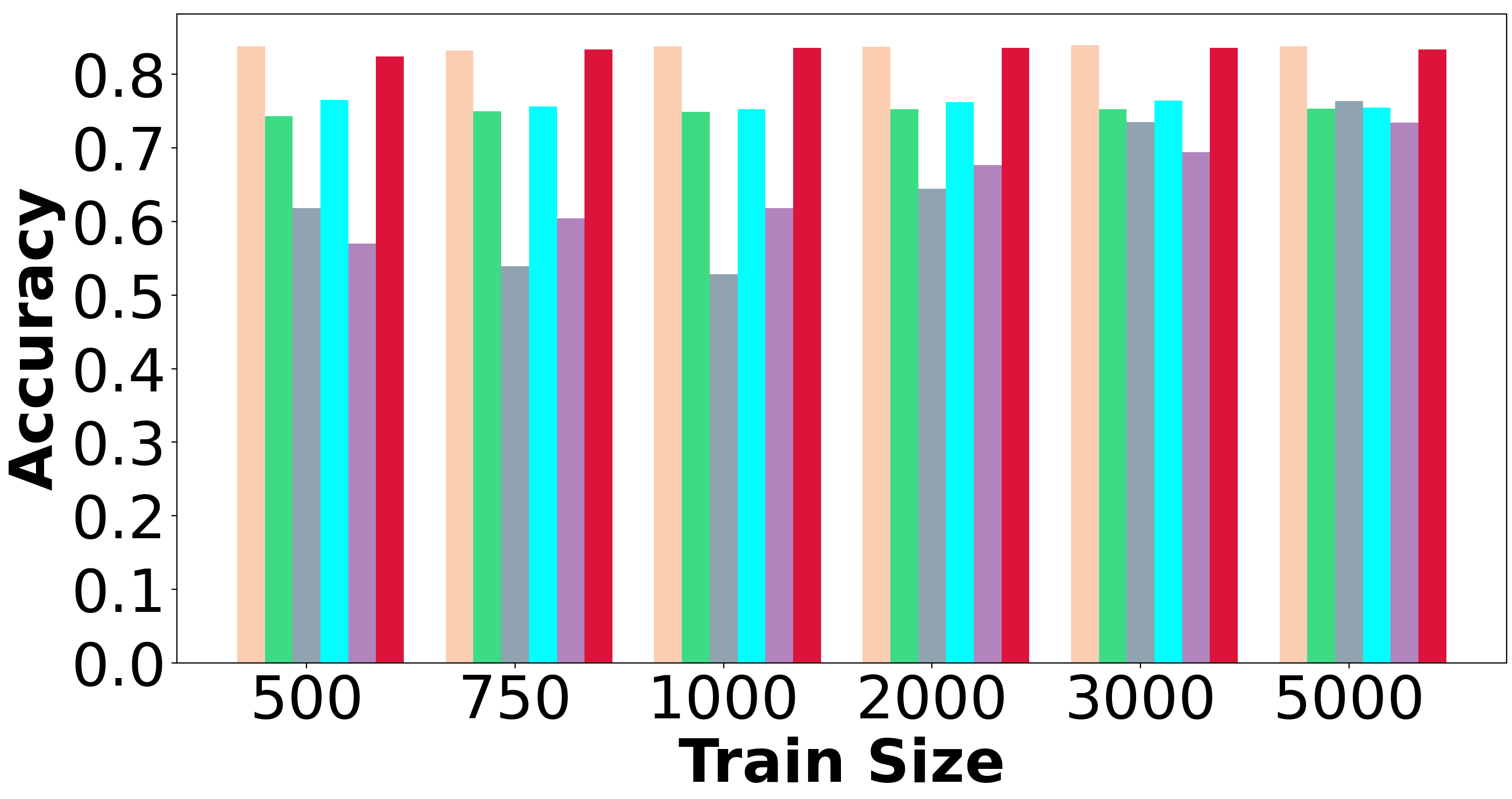}
        \vskip-5pt
        \caption{Learning Identiy Rule  accuracy}\label{fig:exp_ABA}
    \end{subfigure}
    \caption{Experiments on discriminative relational tasks and comparison to SOTA.}
    \vskip-10pt
\end{figure}
\paragraph{Order relations: modeling asymmetric relations.}
\begin{wrapfigure}{R}{0.23\textwidth}
	\vskip-5pt
	\begin{tabular}{c}
		\includegraphics[width=.23\textwidth]{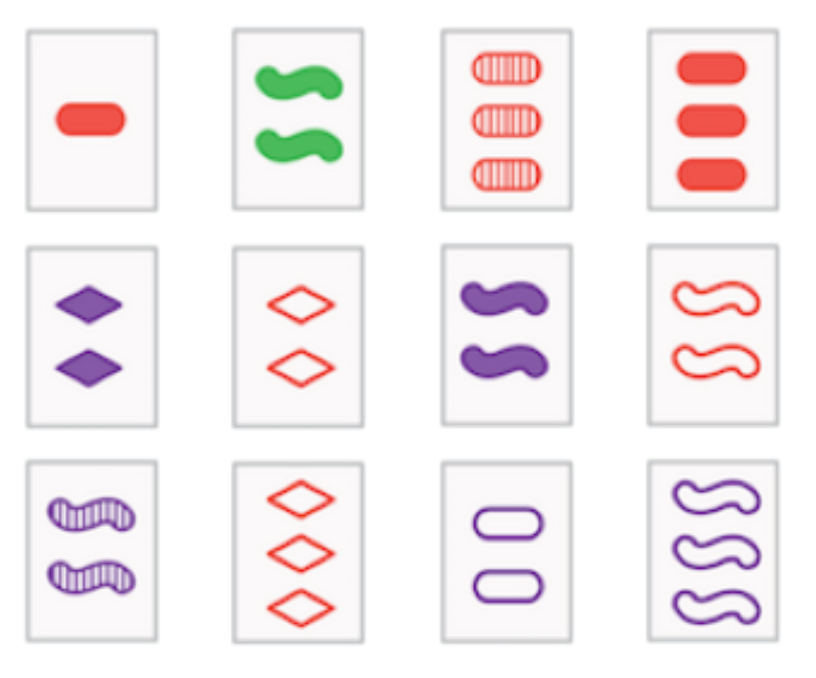}\\[-5pt]
	\end{tabular}
	\caption{\footnotesize The SET game}
 \label{set_figure}
\end{wrapfigure}

We generated 64 random objects represented by iid Gaussian vectors \( o_i \sim \mathcal{N}(0, I) \in \mathbb{R}^{32} \), and established an anti-symmetric order relation \( o_1 \prec o_2 \prec \cdots \prec o_{64} \). From 4096 possible object pairs \((o_i, o_j)\), 15\% are used as a validation set and 35\% as a test set. We train models on varying proportions of the remaining 50\% and evaluate accuracy on the test set, conducting 10 trials for each training set size. The models must generalize based on the transitivity of the \( \prec \) relation using a limited number of training examples. The training sample sizes range between 10 and 200 samples. Figure \ref{fig:exp_order_relation} demonstrates the high capability of LARS-VSA to generalize with few examples, achieving over $80\%$ accuracy with just 200 samples ($1.07$x better than the second best model and $1.33$x better than Abstractor). The Transformer model is the second best performer, better than the Abstractor and CorelNet-Softmax due to the relatively moderate level of abstraction needed for learning asymmetric relations. 

\paragraph{\textit{SET}: modeling multi-dimensional relations.}
In the SET\cite{abstractor} task, players are presented with a sequence of cards, each varying along four dimensions: color, number, pattern, and shape. A triplet of cards forms a "set" if they either all share the same value or each have a unique value (as in Figure \ref{set_figure}). In this experiment, the task is to classify triplets of card images as either a "set" or not.
The shared architecture for processing the card images in all baselines as well as LARS-VSA is $\texttt{CNN} \to \{\cdot\} \to \texttt{Flatten} \to \texttt{Dense}$, where $\{\cdot\}$ is one of the aforementioned modules. The CNN embedder is obtained through a pre-training task. We compared the LARS-VSA architecture to CorelNet \cite{kerg2022neural} with a softmax and a ReLU activation function, PrediNet \cite{shanahan2020explicitly}, Abstractor \cite{abstractor}, MLP \cite{abstractor}, and a Transformer \cite{vaswani2017attention} architecture. For this specific task, there are four relational representations (e.g., shape, color, etc.) and one abstract rule (whether it is a triplet or not). Figure \ref{fig:exp_set_classification} shows LARS-VSA outperforms all the baselines (more than $1.05$x better than the second best model and $1.11$x better than Abstractor when then training on 6000 samples), as it balances abstract rules with relational representations. In this particular case, PrediNet also shows high accuracy. Its relational vectors are less connected to object-level features than those of Transformers but more than those of the Abstractor. 
\paragraph{\textit{ABA}: Modeling Identity Relations}
\begin{wrapfigure}{R}{0.32\textwidth}
	\vskip-5pt
	\begin{tabular}{c}
		\includegraphics[width=.32\textwidth]{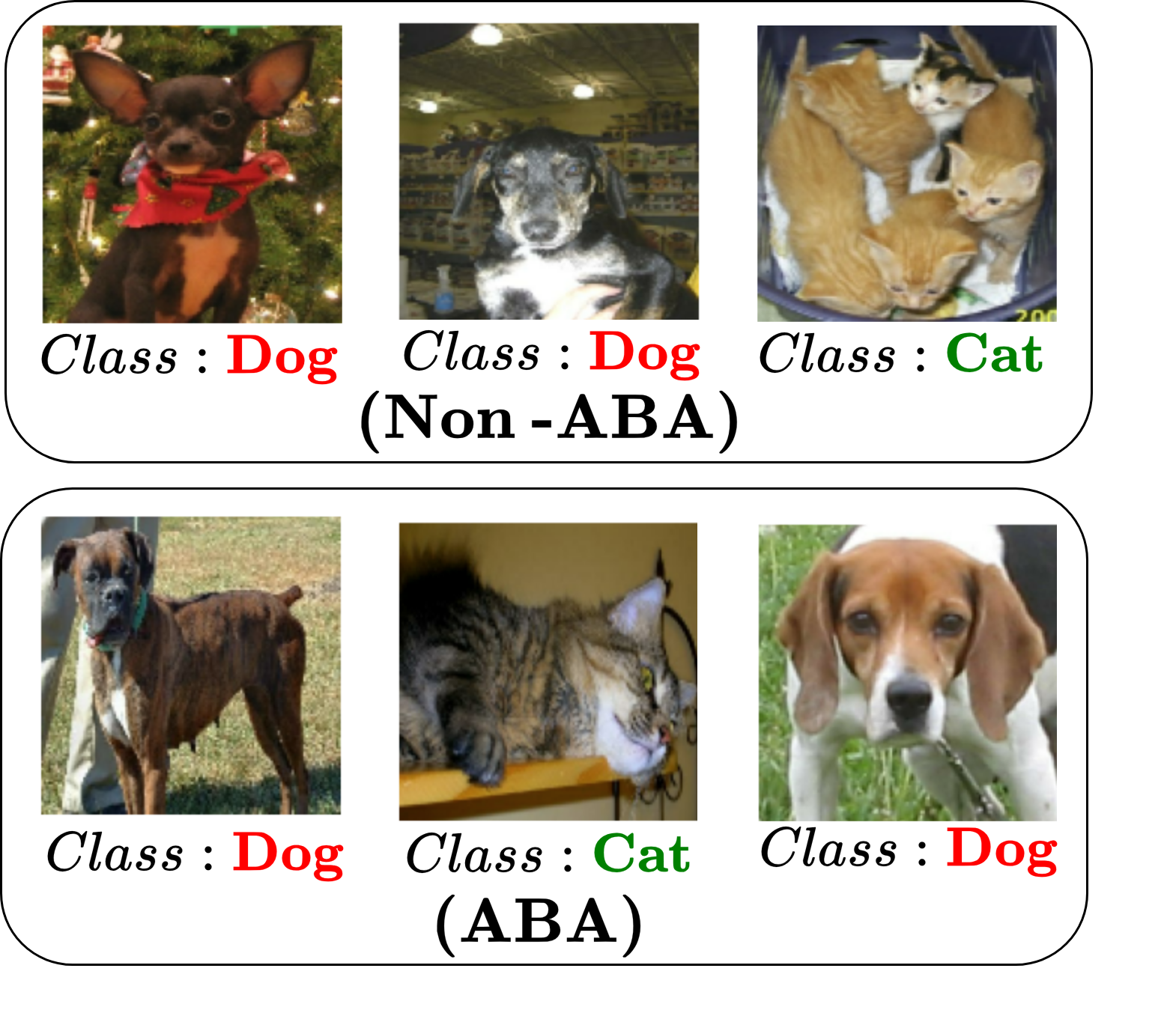}\\[-5pt]
	\end{tabular}
	\caption{\footnotesize ABA Rule Identification Example}
 \vspace{-0.2in}
\end{wrapfigure}
The ABA experiment \cite{webb2024relational} aims to identify an abstract pattern in a sequence of three images from an image classification dataset, specifically the cats and dogs dataset \cite{gong2012computer}. The ABA rule requires the system to determine whether the sequence follows a "cat-dog-cat" or "dog-cat-dog" pattern. The images are processed using a pretrained VGG-like CNN (detailed in the Appendix) to extract the feature map corresponding to the final dense layer. We generated 20000 sequences of 3-image feature maps $o_{i} \in \mathbb{R}^{512} \; \forall \; 1\leq i\leq 3$. The dataset is split into 50\% for training, 25\% for validation and 25\% for testing. In this experiment, Figure \ref{fig:exp_ABA} shows that the Abstractor and LARS-VSA show highest performance (i.e, with comparable accuracy) compared to the state of the art. The difference between this task and the two prior tasks is that the relational representation and the abstract rule (ABA) are harder to correlate since the objects are derived from real images. 

\subsection{Object-sorting: Purely Relational Sequence-to-Sequence Tasks}\label{object_sorting}

We generate random objects as follows. First, we create two sets of random attributes: $\mathcal{A} = {a_1, a_2, a_3, a_4}$, where $a_i \overset{iid}{\sim} \mathcal{N}(0, I) \in \mathbb{R}^{4}$, and $\mathcal{B} = {b_1, \ldots, b_{12}}$, where $b_i \overset{iid}{\sim} \mathcal{N}(0, I) \in \mathbb{R}^{8}$. Each set of attributes has a strict ordering: $a_1 \prec a_2 \prec a_3 \prec a_4$ for $\mathcal{A}$ and $b_1 \prec b_2 \prec \cdots \prec b_{12}$ for $\mathcal{B}$.

Our random objects are formed by taking the Cartesian product of these two sets, $\mathcal{O} = \mathcal{A} \times \mathcal{B}$, resulting in $N = 4 \times 12 = 48$ objects. Each object in $\mathcal{O}$ is a vector in $\mathbb{R}^{12}$, created by concatenating one attribute from $\mathcal{A}$ with one attribute from $\mathcal{B}$.

We then establish a strict ordering relation for $\mathcal{O}$, using the order relation of $\mathcal{A}$ as the primary key and the order relation of $\mathcal{B}$ as the secondary key. Specifically, $(a_i, b_j) \prec (a_k, b_l)$ if $a_i \prec a_k$ or if $a_i = a_k$ and $b_j \prec b_l$. We generated a randomly permuted set of 5 and a set of 6 objects in $\mathcal{O}$. The target sequences are the indices representing the sorted order of the object sequences (i.e., obtained using 'argsort'). The training data are uniformly sampled from the set of length-$N$ (i.e, $N \in \{ 5,6 \}$) sequences in $\mathcal{O}$. Additionally, we generate non-overlapping validation and testing datasets in the following proportion: $20\%$ for testing, $10\%$ for validation and $70\%$ for training.   

We used \textit{element wise accuracy} to assess the performance of LARS-VSA, also used in \cite{abstractor}. The accuracy of LARS-VSA is compared against the Relational Abstractor \cite{abstractor}, Transformer and CorelNet\cite{kerg2022neural}. 
We also examine LARS-VSA using binarized attention that uses a regular binarized attention mechanism instead of our binarized HDSymbolicAttention mechanism ( i.e, instead of evaluating $f({h_{O}}_i,{h_{O}}_i\oplus {h_{O}}_j)$ we simply evaluate $f({h_{O}}_i,{h_{O}}_j)$), to assess the effect of vector orthogonality on LARS-VSA performance. This is referred to as the ablation model in this section.

As the number of elements to sort increases, the performance of all models decreases slightly due to the increasing complexity of the abstract rules. However, LARS-VSA achieves better accuracy than other baselines (between $1.66$x and $2.25$x better than Relational-Abstractor for 5 elements sorting and between $1.56$x and $3.33$x for 6 elements sorting) and the ablation model (around $1.1$x for 5 elements sorting and $1.5$x for 6 elements sorting). Relational-Abstractor \cite{abstractor} still outperforms the transformer and CorelNet-Softmax, validating the results obtained in \cite{abstractor}. LARS-VSA demonstrates greater generalizability compared to the ablation model, suggesting that vector orthogonality in high dimensions has a more significant impact as sequence length increases. This indicates that the attention score matrices of LARS-VSA become less informative as the sequence length increases.

\begin{figure}[ht]
    \centering    \includegraphics[width=0.8\textwidth]{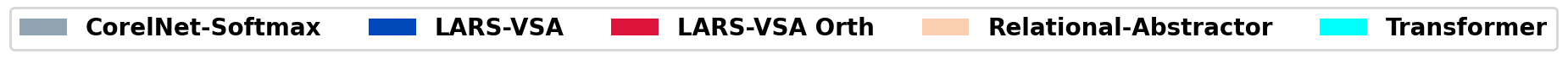}
    \begin{subfigure}[t]{0.35\textwidth}
        \centering\captionsetup{width=.9\linewidth}
        \includegraphics[width=\textwidth]{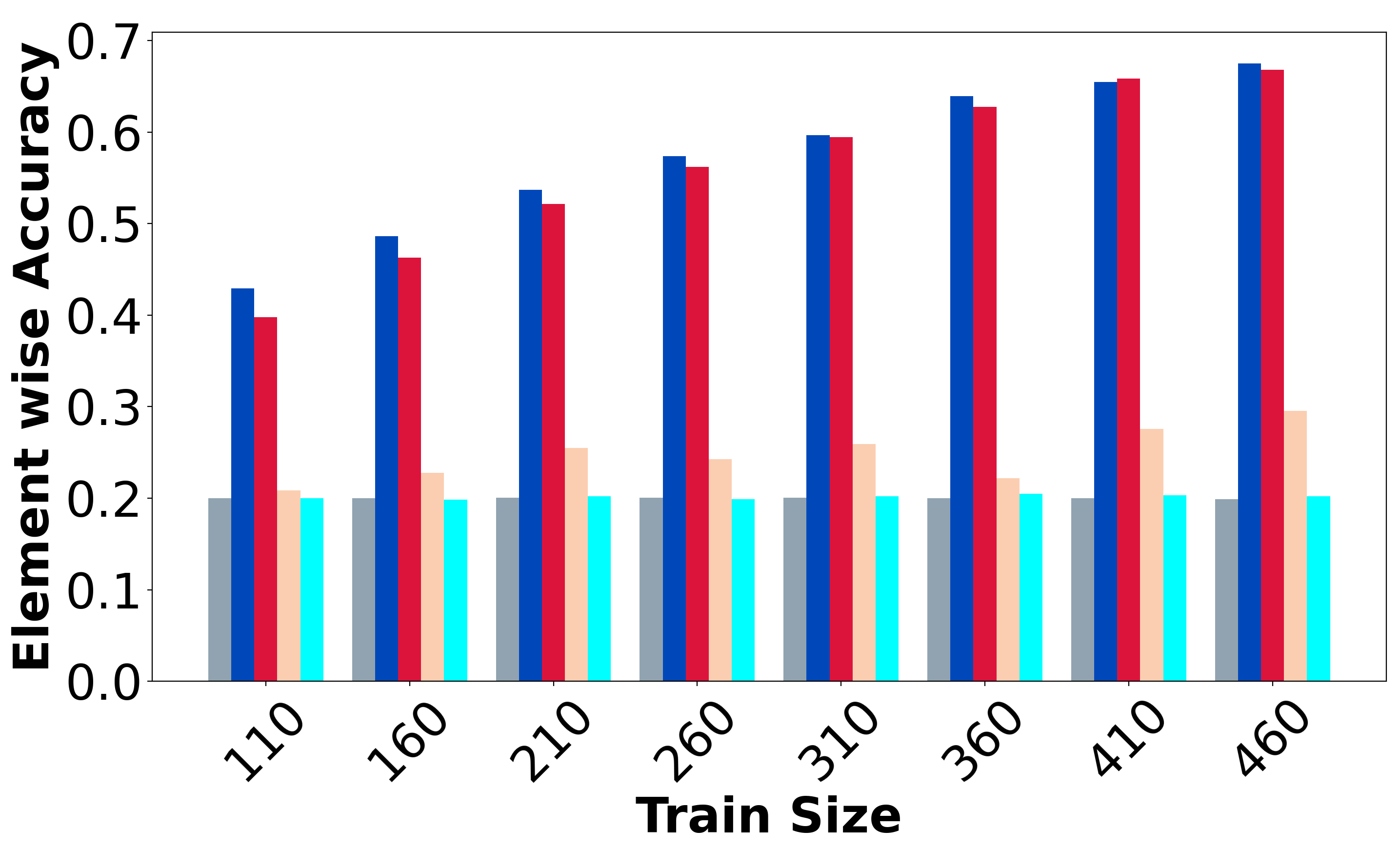}
        \vskip-5pt
        \caption{5 Elements Sequence Sorting}\label{fig:s-5}
    \end{subfigure}
    \begin{subfigure}[t]{0.35\textwidth}
        \centering\captionsetup{width=.9\linewidth}
        \includegraphics[width=\textwidth]{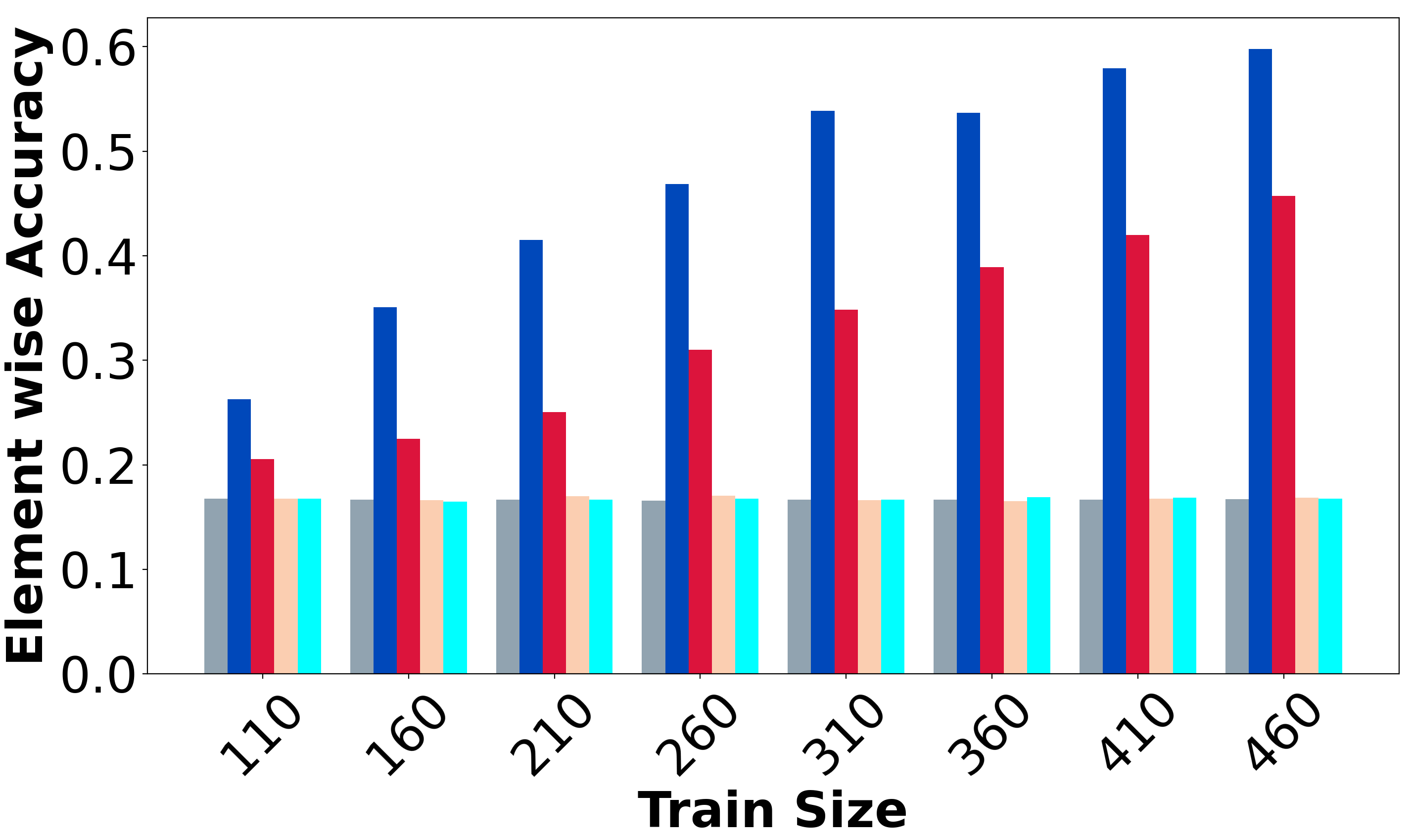}
        \vskip-5pt
        \caption{6 Elements Sequence Sorting}\label{fig:s-6}
    \end{subfigure}
    \caption{Performance of LARS-VSA and LARS-VSA orth compared to baselines. The CorelNet\cite{kerg2022neural} uses \textit{Softmax} activation}
    \vskip-10pt
\end{figure}
\subsection{Math problem-solving: partially-relational sequence-to-sequence tasks}
\begin{figure}[ht]
    \begin{center}
    \begin{small}
    \begin{tabular}{cc}
        \begin{tabular}{l}
        Task: \texttt{numbers\_list\_prime\_factors}\\
        Question: \texttt{What are the prime factors of 121 }\\
        Answer: \texttt{11, 11}
        \end{tabular}
        &
        \begin{tabular}{l}
        Task: \texttt{numbers\_round\_number\_composed}\\
        Question: \texttt{Round 456.789 to 1 decimal place.}\\
        Answer: \texttt{456.8}
        \end{tabular}
    \end{tabular}
    \end{small}
    \end{center}
    \caption{Examples of input/target sequences from the math problem-solving dataset.}\label{fig:math_dataset}
    \vspace{-0.2in}
\end{figure}

The object-sorting experiments in Section \ref{object_sorting} are characterized as "purely relational" because the set of pairwise $\prec$ order relations provides sufficient information to solve the task. For general sequence-to-sequence tasks, there may not always be such a relation. In such cases an architecture with a relational bottleneck may exhibit high abstraction level which might give it an advantage over a purely object level architecture such as a transformer. 
We assessed the LARS-VSA architecture and other baselines on 4 different subset of the mathematical reasoning set of tasks \cite{saxton2019analysing}. Two of them are presented in Figure \ref{fig:math_dataset}. The rest are  described in the appendix.
\begin{figure}[ht!]
    \centering
    \includegraphics[width=0.8\textwidth]{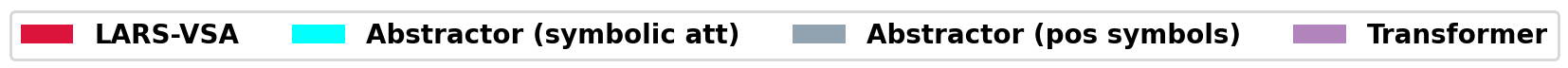}
    \begin{subfigure}[b]{0.42\textwidth}
        \centering
        \includegraphics[width=\textwidth]{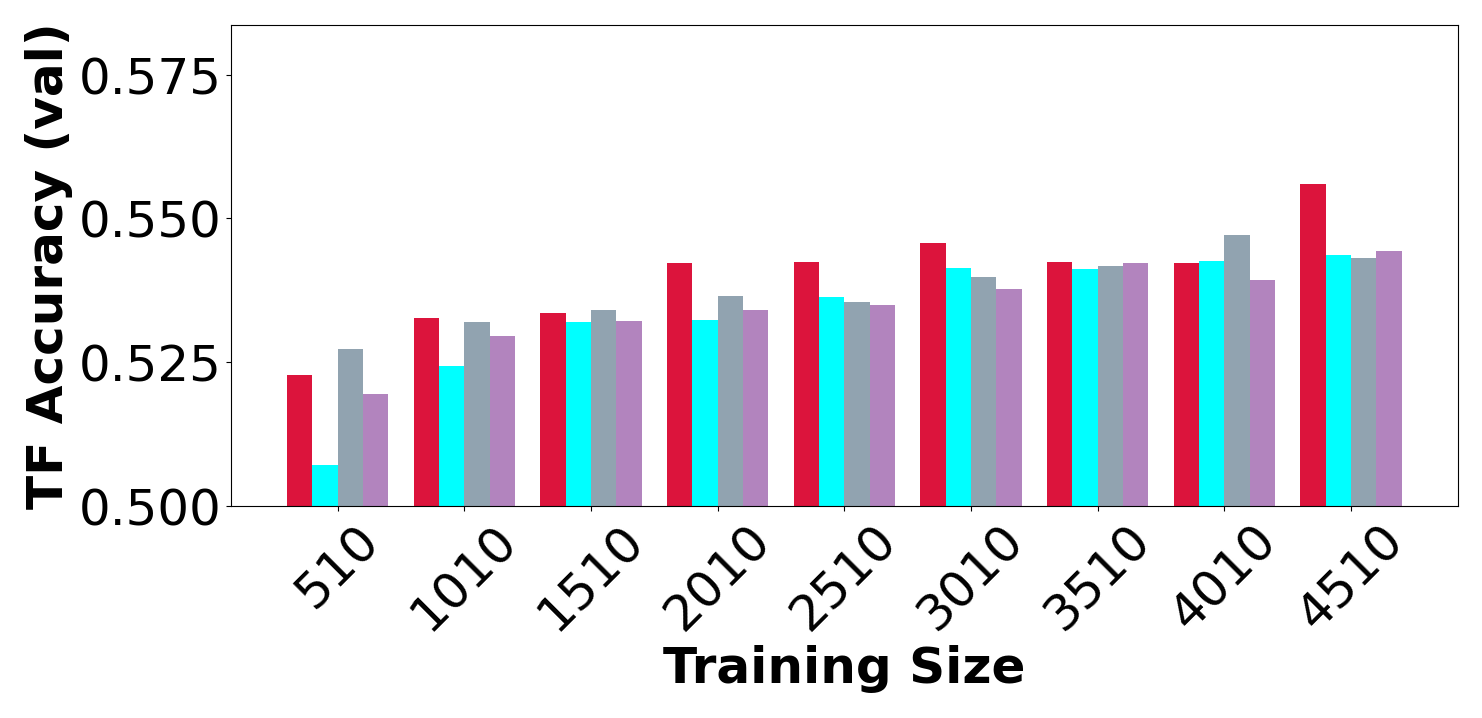}
        \caption{Numbers list prime factors task}
        \label{fig:task1}
    \end{subfigure}
    \hfill
    \begin{subfigure}[b]{0.42\textwidth}
        \centering
        \includegraphics[width=\textwidth]{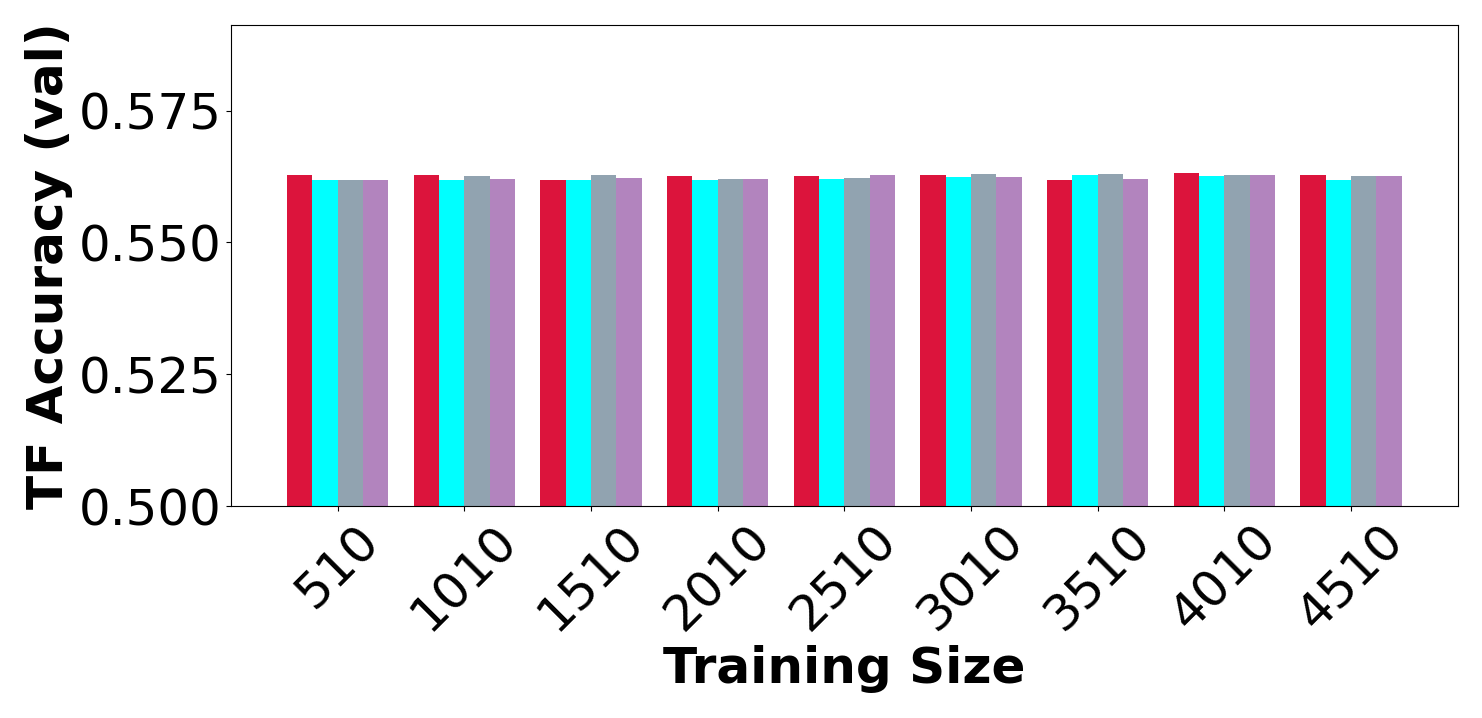}
        \caption{Numbers place value composed task}
        \label{fig:task2}
    \end{subfigure}
    \label{fig:total_math}
    \centering
    \caption{Accuracy of LARS-VSA compared to SOTA for two mathematical reasoning tasks}
\end{figure}
\begin{figure}[ht!]
    \centering
    \begin{subfigure}[b]{0.3\textwidth}
        \centering
        \includegraphics[width=\textwidth]{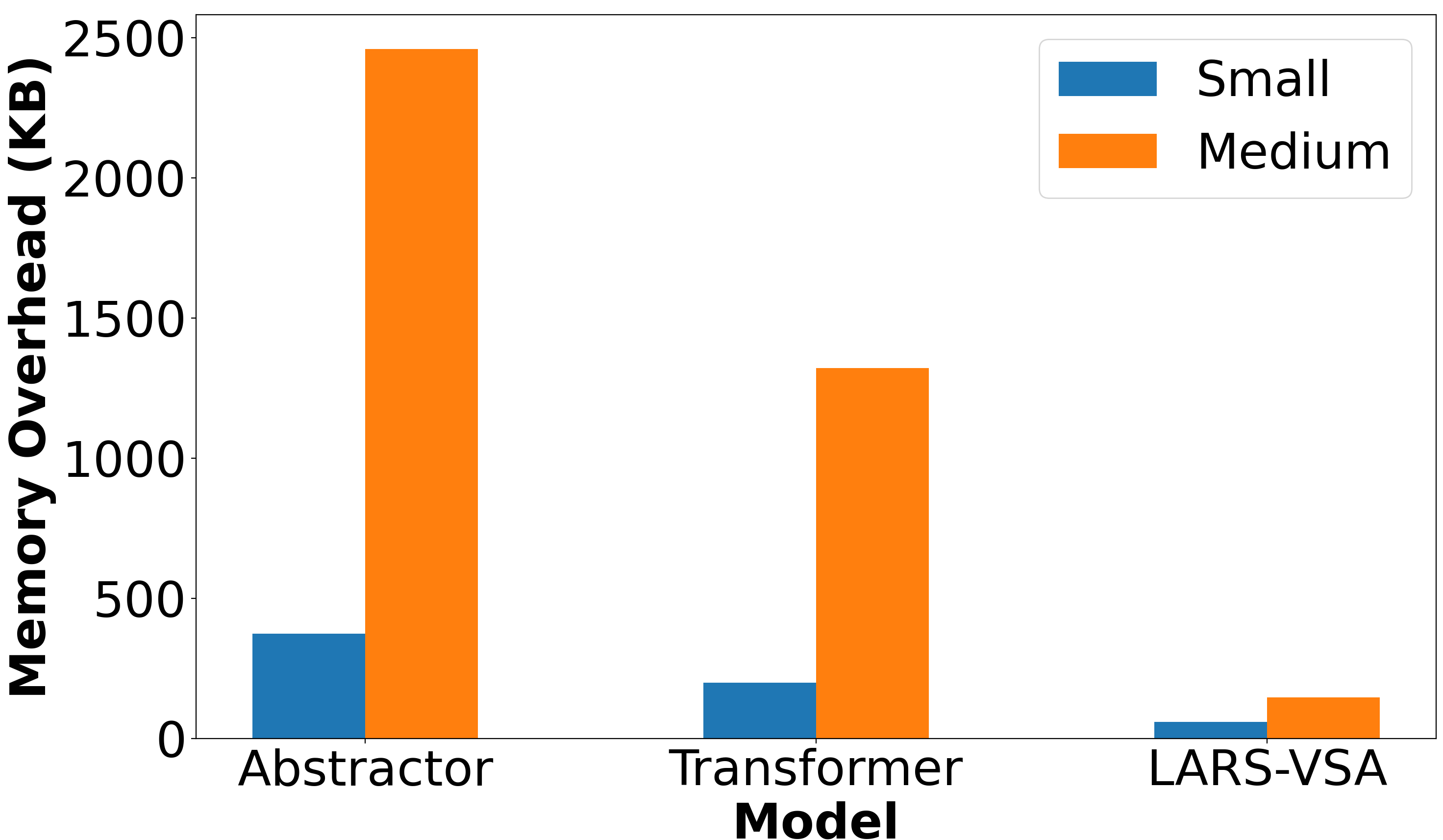}
        \caption{\centering Memory size of different models}
        \label{fig:mem}
    \end{subfigure}
    \hfill
    \begin{subfigure}[b]{0.4\textwidth}
        \centering
        \includegraphics[width=\textwidth]{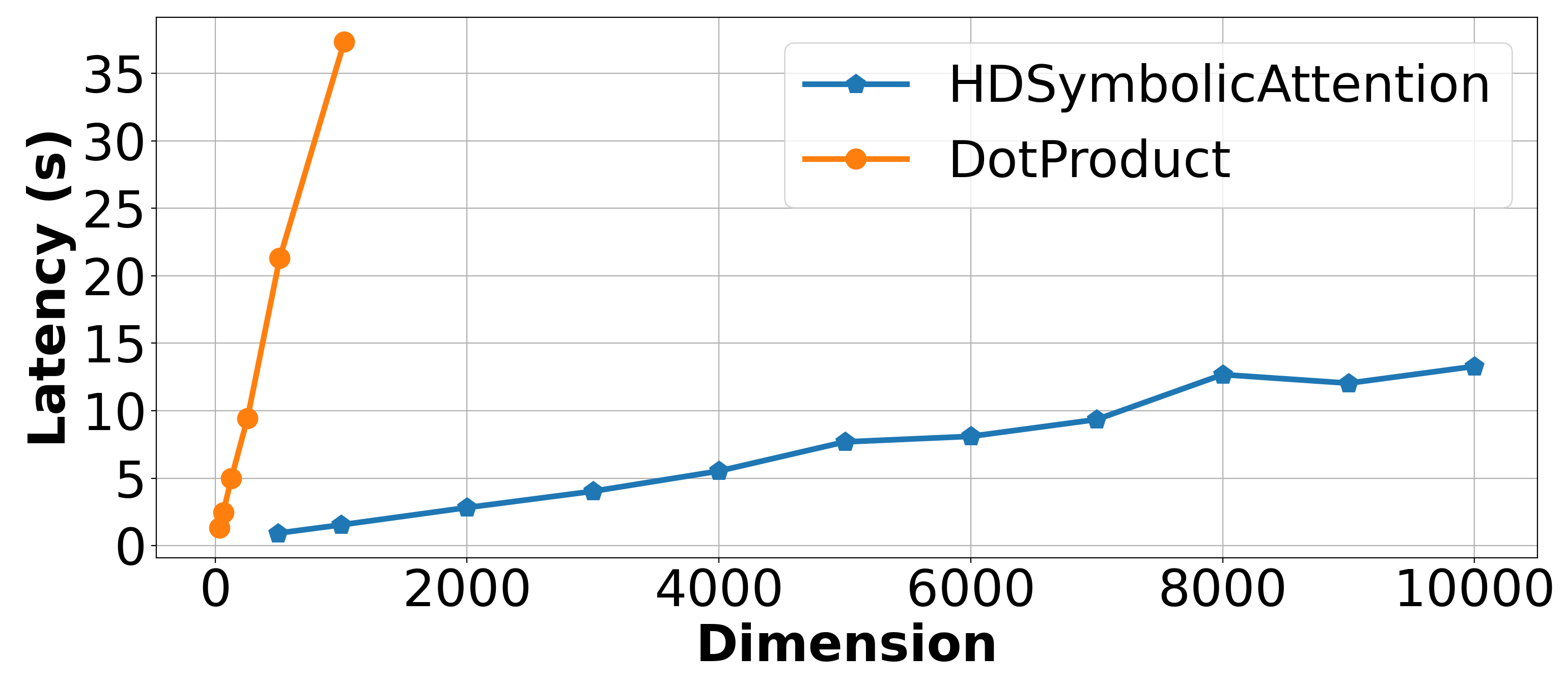}
        \caption{\centering Latency of dot-product Vs binarized cosine similarity for $10^{7}$ iterations}
        \label{fig:lat}
    \end{subfigure}
    \label{fig:namename}
    \centering
    \caption{Memory overhead of LARS-VSA, Abstractor and Transformer Figure (a) and Latency of dot-product compared to binarized cosine similarity Figure (b)}
    \vspace{-0.2in}
\end{figure}
The baselines used to assess the LARS-VSA performance are: two version of Abstractor with symbolic relational attention\cite{abstractor} and positional symbols\cite{abstractor}, and a Transformer architecture. All architectures use a "small" Encoder-Decoder structure except LARS-VSA, which has a lightweight encoder. For the first task (Figure \ref{fig:task1}), LARS-VSA outperforms other models by up to $4\%$ in accuracy. For the second task (Figure \ref{fig:task2}), all models have comparable accuracies. LARS-VSA excels in handling complex relational tasks with minimal overhead.
We compared LARS-VSA's speedup and memory savings to a regular transformer and Abstractor for the mathematical reasoning task. Figure \ref{fig:mem} shows the memory overhead in bits for Abstractor (symbolic attention), Transformer, and LARS-VSA using small and medium Encoder-Decoder structures. Figure \ref{fig:lat} illustrates the latency of one elementary operation for self-attention (dot product) and HDSymbolicAttention (binarized attention score). LARS-VSA is up to 17x and 9x more memory efficient than Abstractor and Transformer with medium Encoder-Decoder structures, and up to 6x and 3x more efficient with small Encoder-Decoder architectures.
From Figure \ref{fig:lat}, HDSymbolicAttention and Dot product grow linearly with rates of $0.14 \times 10^{-2}$ and $3.6 \times 10^{-2}$, respectively, indicating HDSymbolicAttention is about 25x faster across all latency-to-vector dimension ratios.

\newpage
\section{Conclusion, Limitation and Future Directions}
In this paper, we introduce a novel \textit{relational bottleneck} for vector symbolic architectures, utilizing a hyperdimensional computing-specific attention mechanism called \textit{HDSymbolicAttention}. This mechanism binds correlated object feature high-dimensional vectors with symbolic abstract high-dimensional vectors, reducing interference between object-level features and abstract rules, effectively addressing the \textit{curse of compositionality} problem. However, interference is not only spatial but also temporal, as relational representations can interfere with symbolic rules over time, potentially leading to \textit{catastrophic forgetting}, a common issue in online-trained neural networks. Future work will explore the impact of this structure on online abstract reasoning abilities and develop a decoder (i.e, see Figure\ref{fig:architecture}) that relies solely on hyperdimensional computing capabilities, moving away from the current cost-inefficient self-attention and cross-attention mechanisms.
\bibliography{ref}
\bibliographystyle{plainnat}


\appendix

\section{Appendix / supplemental material}

\subsection*{Code and Reproducibility}
Code, detailed experimental logs, and instructions for reproducing our experimental results are available at: \url{https://github.com/mmejri3/LARS-VSA}.
\subsection*{Discriminative Tasks}
In this section, we provide comprehensive information on the architectures, hyperparameters, and implementation specifics of our experiments. All models and experiments were developed using TensorFlow. The code, along with detailed experimental logs and reproduction instructions, is available in the project's public repository.
\subsection{Computational Resources}
For training the LARS-VSA and SOTA on the Discriminative relational tasks we used a CPU (11th Gen Intel® Core™ i7) For training LARS-VSA and SOTA on the purely and partially sequence to sequence abstract reasoning we used a cluster of 4 GPUs RTX4090 with 24GB of RAM each.    
\subsection{\textit{Discriminative Tasks}}

\subsubsection{Pairwise Order}
Each model in this experiment follows the structure: \texttt{input} $\to$ \{\texttt{module}\} $\to$ \texttt{flatten} $\to$ \texttt{MLP}, where \{\texttt{module}\} represents one of the described modules and \texttt{MLP} is a multilayer perceptron with one hidden layer of 32 neurons activated by ReLU.

\paragraph{LARS-VSA Architecture}
Each model is composed of a single head $H=1$ and a hypervector of dimensionality equal to $D=1000$. We use a dropout of $0.1$ to avoid overfitting. The two hypervectors are flattened and passed to the hidden layers with 32 neurons with a ReLU activation. It is forwarded to a 1 neuron final layer with sigmoid activation.

\paragraph{Abstractor Architecture}
The Abstractor module utilizes the following hyperparameters: number of layers $L = 1$, relation dimension $d_r = 4$, symbol dimension $d_s = 64$, projection (key) dimension $d_k = 16$, feedforward hidden dimension $d_{\mathrm{ff}} = 64$, relation activation function $\sigma_{\mathrm{rel}} = \mathrm{softmax}$. No layer normalization or residual connection is applied. Positional symbols, which are learned parameters of the model, are used as the symbol assignment mechanism. The output of the Abstractor module is flattened and passed to the \texttt{MLP}.

\paragraph{CoRelNet Architecture}
CoRelNet has no hyperparameters. Given a sequence of objects, $X = (x_1, \ldots, x_m)$, standard CoRelNet~\citep{kerg2022neural} computes the inner product and applies the Softmax function. We also add a learnable linear map, $W \in \mathbb{R}^{d \times d}$. Hence, $\bar{R} = \text{Softmax}(R), R = {\left[\langle W x_i, W x_j\rangle\right]}_{ij}$. The CoRelNet architecture flattens $\bar{R}$ and passes it to an \texttt{MLP} to produce the output. The asymmetric variant of CoRelNet is given by $\bar{R} = \text{Softmax}(R), R = {\left[\langle W_1 x_i, W_2 x_j\rangle\right]}_{ij}$, where $W_1, W_2 \in \mathbb{R}^{d \times d}$ are learnable matrices.

\paragraph{PrediNet Architecture}
Our implementation of PrediNet~\citep{shanahan2020explicitly} is based on the authors' publicly available code. We used the following hyperparameters: 4 heads, 16 relations, and a key dimension of 4 (see the original paper for the meaning of these hyperparameters). The output of the PrediNet module is flattened and passed to the \texttt{MLP}.

\paragraph{MLP}
The embeddings of the objects are concatenated and passed directly to an \texttt{MLP}. The \texttt{MLP} has two hidden layers, each with 32 neurons and a ReLU activation.

\paragraph{Training/Evaluation}
We use the cross-entropy loss and the AdamW optimizer with a learning rate of $10^{-4}$, We use a batch size of 64. Training is conducted for 50 epochs. The evaluation is performed on the test set. The experiments are repeated 10 times and we take the mean accuracy and will report the standard deviation later.  

\subsubsection{\textit{SET}}

The card images are RGB images with dimensions of $70 \times 50 \times 3$. A CNN embedder processes these images individually, producing embeddings of dimension $d=64$ for each card. The CNN is trained to predict four attributes of each card, after which embeddings are extracted from an intermediate layer and the CNN parameters are frozen. The common architecture follows the structure: \texttt{CNN Embedder} $\to$ \{\texttt{Abstractor, CoRelNet, PrediNet, MLP}\} $\to$ \texttt{Flatten} $\to$ \texttt{Dense(2)}. Initial tests with the standard CoRelNet showed no learning. By adjusting hyperparameters and architecture, we found that removing the softmax activation in CoRelNet improved its performance slightly. Details of hyperparameters are provided below.

\textbf{Common Embedder Architecture:} 
The architecture is structured as: \texttt{Conv2D} $\to$ \texttt{MaxPool2D} $\to$ \texttt{Conv2D} $\to$ \texttt{MaxPool2D} $\to$ \texttt{Flatten} $\to$ \texttt{Dense(64, relu)} $\to$ \texttt{Dense(64, relu)} $\to$ \texttt{Dense(2)}. The embedding is taken from the penultimate layer. The CNN is trained to perfectly predict the four attributes of each card, achieving near-zero loss.

\textbf{LARS-VSA Architecture:}
The LARS-VSA module has the following hyperparameters: number of heads $H = 1$, hypervector dimension $D = 1000$. The output are flattened and then feedforward hidden dimension $d_{\mathrm{ff}} = 64$ and then tp a final layer with a single neuron with $sigmoid$ activation. We used a dropout of $0.4$ to avoid overfitting.

\textbf{Abstractor Architecture:}
The Abstractor module has the following hyperparameters: number of layers $L = 1$, relation dimension $d_r = 4$, symmetric relations ($W_q^{i} = W_k^{i}$ for $i \in [d_r]$), a relation activation ReLU, symbol dimension $d_s = 64$, projection (key) dimension $d_k = 16$, feedforward hidden dimension $d_{\mathrm{ff}} = 64$, and no layer normalization or residual connection. Positional symbols, which are learned parameters, are used as the symbol assignment mechanism.

\textbf{CoRelNet Architecture:}
The standard CoRelNet, described above, computes $R = \text{Softmax}(A)$, where $A = {\left[\langle W x_i, W x_j\rangle\right]}_{ij}$. This variant was stuck at 50\% accuracy regardless of the training set size. Removing the softmax improved performance. Figure \ref{fig:exp_set_classification} compares both variants of CoRelNet.

This finding indicates that making $\sigma_{\mathrm{rel}}$ a configurable hyperparameter is beneficial. Softmax normalizes relations contextually, which might be useful at times but can hinder relational models when absolute relations between object pairs are needed, independent of other relations.

\textbf{PrediNet Architecture:}
We used the following hyperparameters: 4 heads, 16 relations, and a key dimension of 4 (as per the original paper). The output of the PrediNet module is flattened and passed to the MLP.

\textbf{MLP:}
The embeddings of the objects are concatenated and passed directly to an MLP. The MLP has two hidden layers, each with 32 neurons and ReLU activation.

\textbf{Data Generation:}
Data is generated by randomly sampling a ``set'' with probability 1/2 and a non-``set'' with probability 1/2. The triplet of cards is then randomly shuffled.

\textbf{Training/Evaluation:}
We use cross-entropy loss and the AdamW optimizer with a learning rate of $10^{-4}$, The batch size is 512. Training is conducted for 200 epochs. Evaluation is performed on the test set. We train our model on a set of randomly N number of samples $N \in \{ 1000,2000,3000,4000,5000,6000\}$

\subsubsection{\textit{ABA}}
\begin{figure}[ht]
    \centering
    \includegraphics[width=0.9\textwidth]{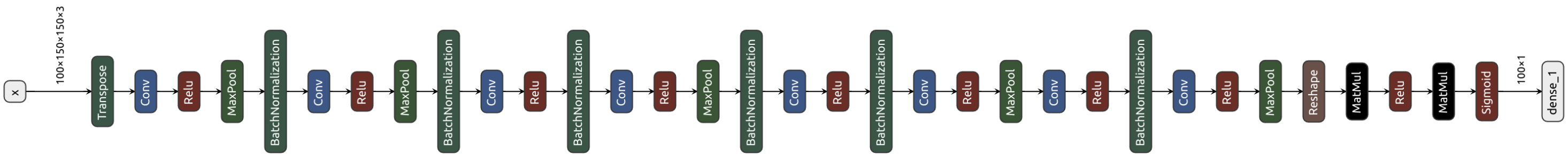}
    \caption{CNN Feature Extractor}
    \label{fig:cnn_extractor}
\end{figure}
This task consists of an Identity rule ABA, so the input is a sequence of 3 vectors of shape 512. The class of each vector is either A or B. If the sequence is ordered in an ABA order the relational model should label it as 1 otherwise as 0. The dataset used is CatVsDog with an 20000 images in total. The CNN used as a Feature extractor is shown in Figure\ref{fig:cnn_extractor}. It is trained over 100 epochs with an Adam optimized with $lr=10^{-3}$. The last layer of the CNN is taken as a feature map of shape 512. We trained the relational models on N samples with $N \in \{ 500, 750, 1000, 2000, 3000, 5000 \}$
We used the same baseline settings as \textit{SET} experiment. 

\subsection{Relational sequence-to-sequence tasks}
\subsubsection{\textit{Object-sorting task}}

\paragraph{LARS-VSA Architecture}
We used the architecture (B) of Figure \ref{fig:architecture}. The encoder used in a BatchNormalization layer. The LARS-VSA model includes 2 heads $H=2$ with a hyperdimensional dimension of $D=1000$. The decoder used 4 layers, 2 attention heads, a feedforward network with 64 hidden units, and a model dimension of 64.

\paragraph{Abstractor Architecture}
Each of the Encoder, Abstractor, and Decoder modules consists of $L = 2$ layers, with 2 attention heads/relation dimensions, a feedforward network with $d_{\mathrm{ff}} = 64$ hidden units, and a model/symbol dimension of $d_{\mathrm{model}} = 64$. The relation activation function is $\sigma_{\mathrm{rel}} = \mathrm{Softmax}$. Positional symbols are utilized as the symbol assignment mechanism, which are learned parameters of the model. 

\paragraph{Transformer Architecture}
We implement the standard Transformer architecture as described by \citep{vaswani2017attention}. Both the Encoder and Decoder modules use matching hyperparameters per layer, with an increased number of layers. Specifically, we use 4 layers, 2 attention heads, a feedforward network with 64 hidden units, and a model dimension of 64.

\paragraph{Training and Evaluation}
The models are trained using cross-entropy loss and the Adam optimizer with a learning rate of $5\cdot{10^{-4}}$. We use a batch size of 64 and train for 150 epochs. To evaluate learning curves, we vary the training set size, sampling random subsets ranging from 110 to 460 samples in increments of 50. Each sample consists of an input-output sequence pair. For each model and training set size, we perform 10 runs with different random seeds, reporting the mean of accuracy.

\subsubsection{\textit{Math Problem-Solving}}

The dataset comprises various math problem-solving tasks, each featuring a collection of question-answer pairs. These tasks cover areas such as solving equations, expanding polynomial products, differentiating functions, predicting sequence terms, and more. Figure \ref{fig:total_math_2} provides an example of these question-answer pairs. The dataset includes 2 million training examples and 10,000 validation examples per task. Questions are limited to a maximum length of 160 characters, while answers are restricted to 30 characters. Character-level encoding is utilized, employing a shared alphabet of 95 characters, which includes upper and lower case letters, digits, punctuation, and special tokens for start, end, and padding.

\paragraph{Abstractor Architectures.} The Encoder, Abstractor, and Decoder modules share identical hyperparameters: number of layers $L = 1$, relation dimension/number of heads $d_r = n_h = 2$, symbol dimension/model dimension $d_s = d_{\mathrm{model}} = 64$, projection (key) dimension $d_k = 32$, feedforward hidden dimension $d_{\mathrm{ff}} = 128$. The relation activation function in the Abstractor is $\sigma_{\mathrm{rel}} = \mathrm{softmax}$. One model employs positional symbols with sinusoidal embeddings, while the other utilizes symbolic attention with a symbol library of $n_s = 128$ learned symbols and $2$-head symbolic attention.

\paragraph{Transformer Architecture.} The Transformer Encoder and Decoder possess hyperparameters identical to those of the Encoder and Decoder in the Abstractor architecture.

\paragraph{LARS-VSA Architectures.}
The LARS-VSA uses the architecture (C) of Figure\ref{fig:architecture}. We used the same Decoder as the Abstractor architecture. However, the encoder structure is limited to a BatchNormalization layer. The LARS-VSA has 2 heads $H=2$ and a hyperdimensional  dimension of $D = 1000$.

\paragraph{Training/Evaluation.} Each model is trained for 150 epochs using categorical cross-entropy loss and the Adam optimizer with a learning rate of $6 \times 10^{-4}$, $\beta_1 = 0.9$, $\beta_2 = 0.995$, and $\varepsilon = 10^{-9}$. The batch size used is 16. The training set is composed of N samples $N \in \{ 510, 1010, 1510, 2010, 2510, 3010, 3510, 4010, 4510 \}$. 
\subsection{Additional Results}
\paragraph{Standard deviation discriminative relational tasks}
for the Descriminative relational tasks (pairwise ordering, SET, ABA) we repeated the experiment 10 times. In the previous section we presented the mean value, In the Figure \ref{fig:std} we presented the standard deviation of the mean value. 

\begin{figure}[ht]
    \centering
    \includegraphics[width=0.8\textwidth]{figures/experiments/legend_disc.png}
    \begin{subfigure}[t]{0.3\textwidth}
        \centering\captionsetup{width=.9\linewidth}
        \includegraphics[width=\textwidth]{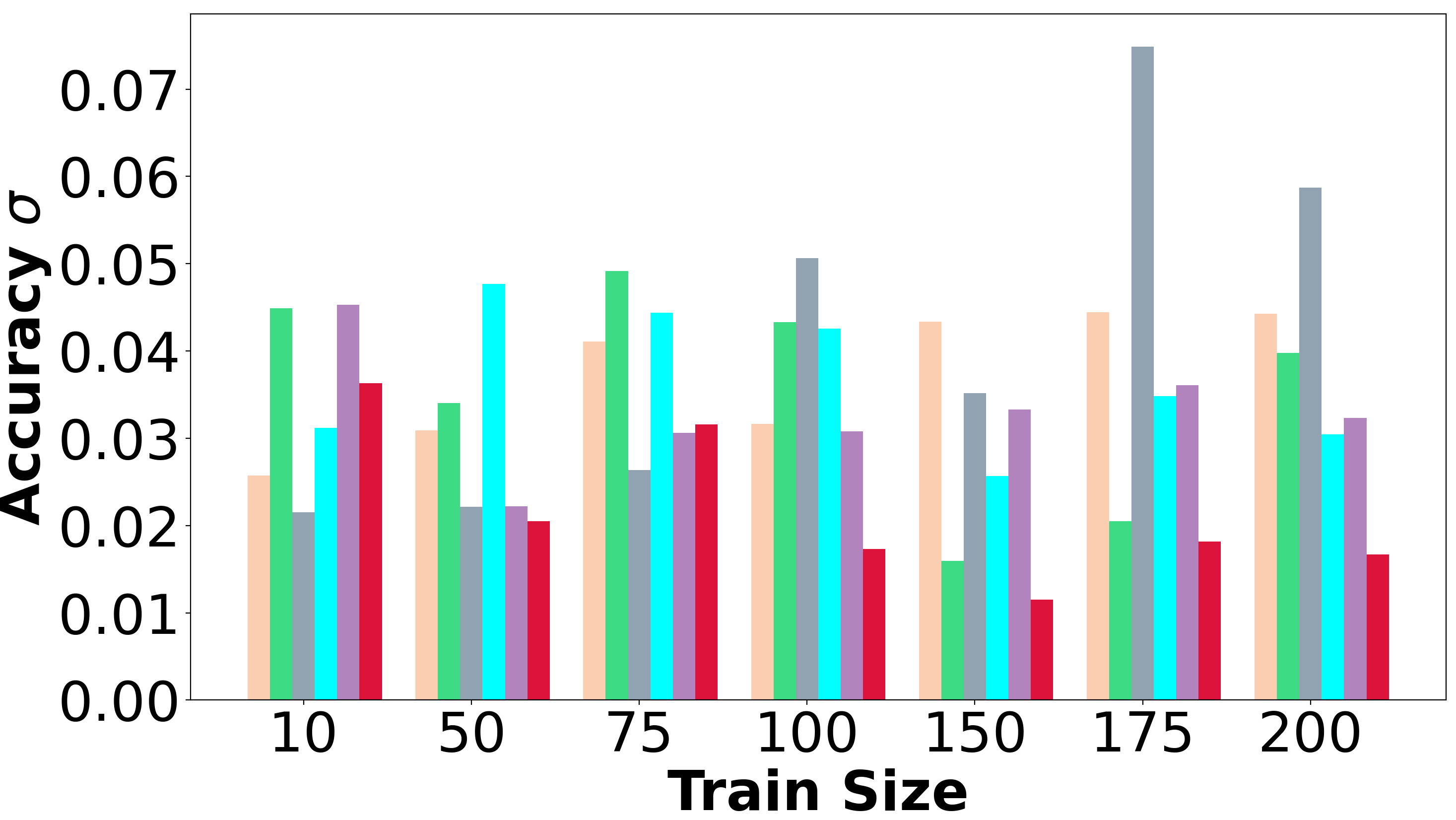}
        \vskip-5pt
        \caption{Learning $\prec$ relation  standard deviation}\label{fig:exp_order_relation_std}
    \end{subfigure}
    \begin{subfigure}[t]{0.32\textwidth}
        \centering\captionsetup{width=.9\linewidth}
        \includegraphics[width=\textwidth]{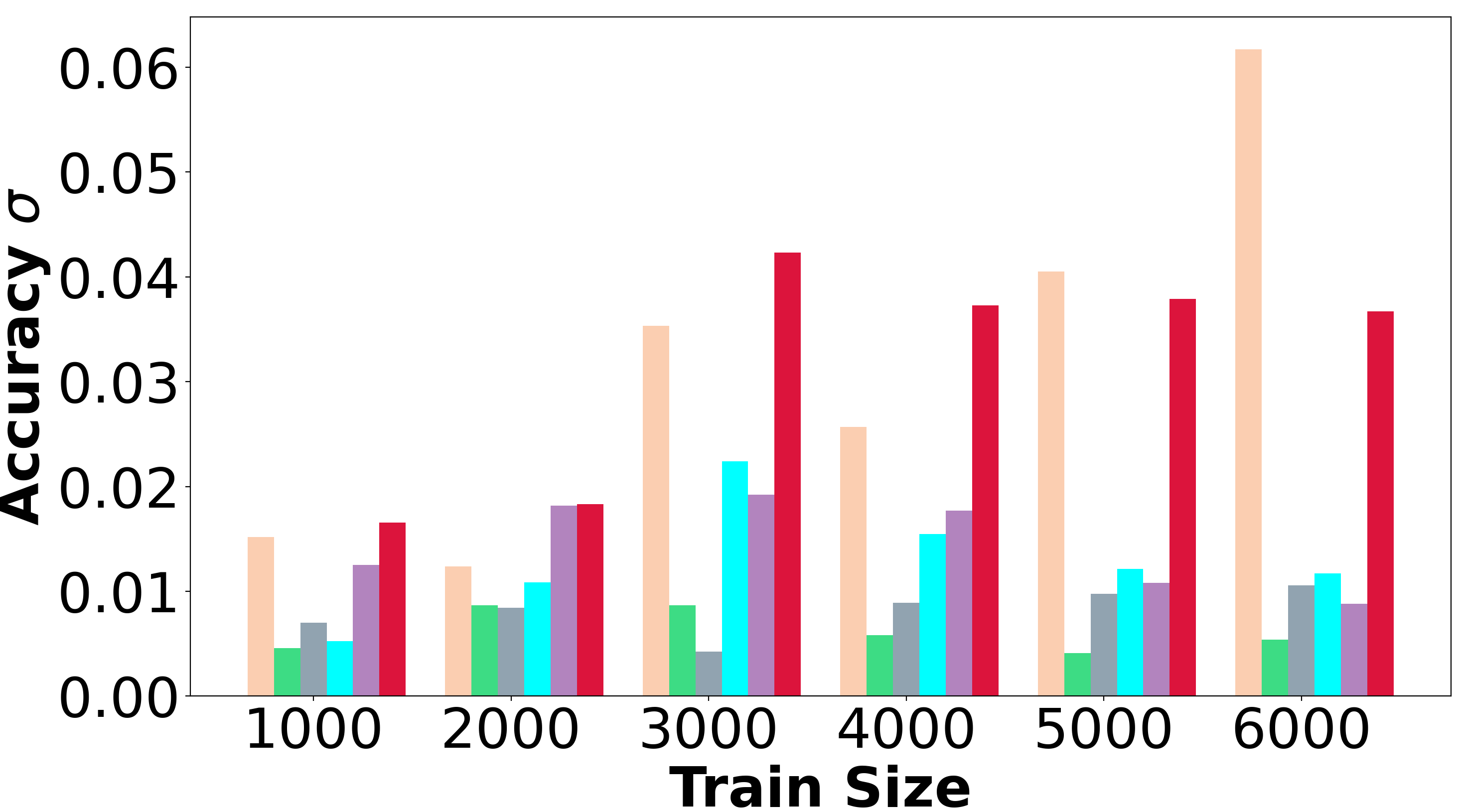}
        \vskip-5pt
        \caption{SET Classification standard deviation}\label{fig:exp_set_classification_std}
    \end{subfigure}
    \begin{subfigure}[t]{0.32\textwidth}
        \centering\captionsetup{width=.9\linewidth}
        \includegraphics[width=\textwidth]{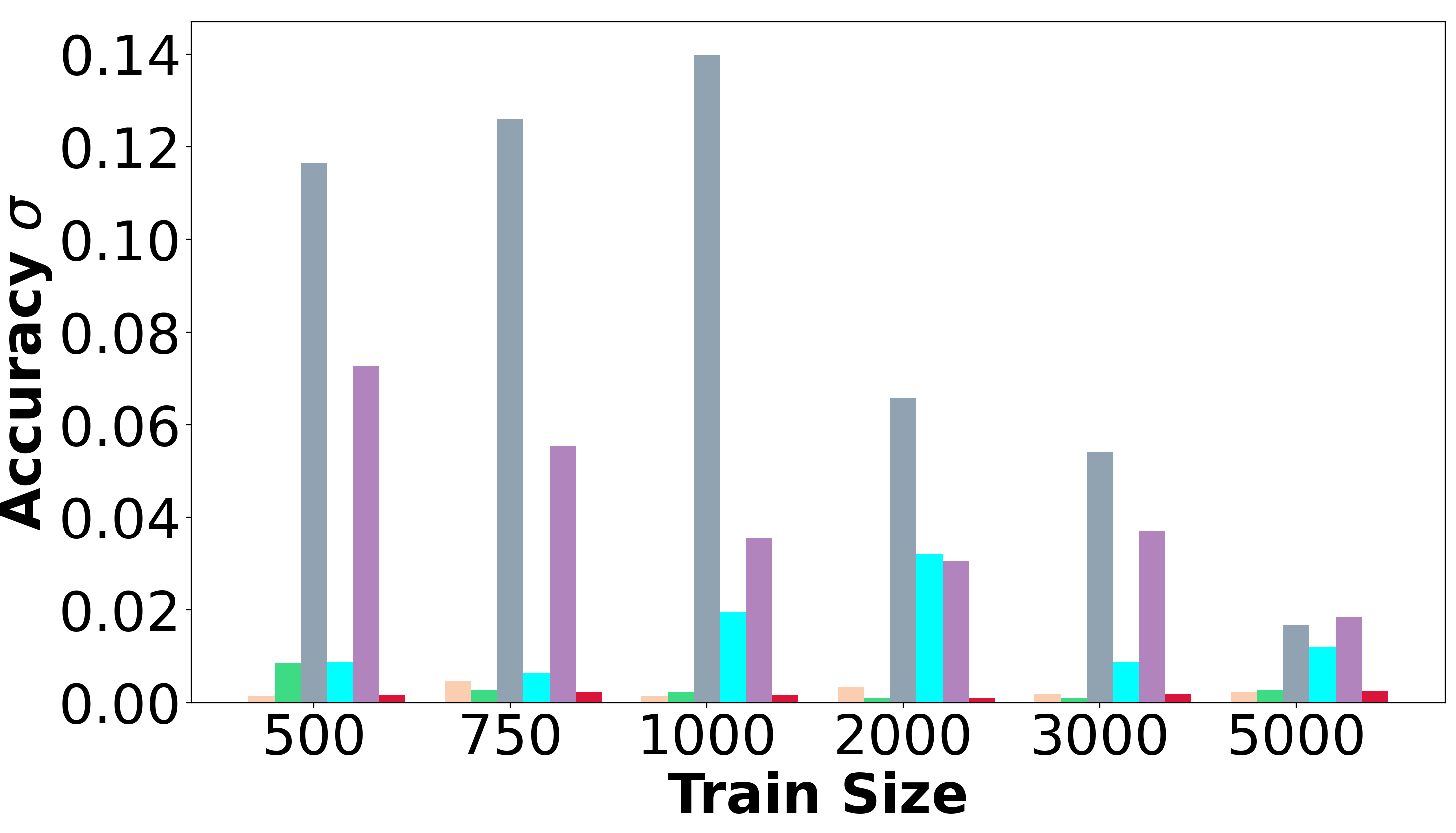}
        \vskip-5pt
        \caption{Learning Identiy Rule  standard deviation}\label{fig:exp_ABA_std}
    \end{subfigure}
    \caption{Standard Deviation of experiments on discriminative relational tasks and comparison to SOTA.}
    \label{fig:std}
\end{figure}
It turns out from Figures \ref{fig:exp_ABA_std} and \ref{fig:exp_order_relation_std} that LARS-VSA showed the highest robustness to random seed proving its stability to randomness. For Figure \ref{fig:exp_set_classification_std}, LARS-VSA showed higher standard deviation compared to SOTA but it's still low ($\sigma \leq 0.06$)  
\paragraph{Mathematical reasoning tasks}
\begin{figure}[ht]
    \begin{center}
    \begin{small}
    \begin{tabular}{cc}
        \begin{tabular}{l}
        Task: \texttt{algebra\_\_linear\_2d}\\
        Question: \texttt{Solve for x and y: 3x + 2y = 6 and x - y = 2}\\
        Answer: \texttt{x = 2, y = 0}
        \end{tabular}
        &
        \begin{tabular}{l}
        Task: \texttt{numbers\_\_is\_prime}\\
        Question: \texttt{Is 97 a prime number?}\\
        Answer: \texttt{Yes}
        \end{tabular}
    \end{tabular}
    \end{small}
    \end{center}
    \caption{Examples of input/target sequences from the math problem-solving dataset.}\label{fig:math_dataset2}
    \vspace{-0.2in}
\end{figure}

\begin{figure}[ht]
    \centering
    \includegraphics[width=0.8\textwidth]{figures/experiments/legend_math.png}
    \begin{subfigure}[b]{0.42\textwidth}
        \centering
        \includegraphics[width=\textwidth]{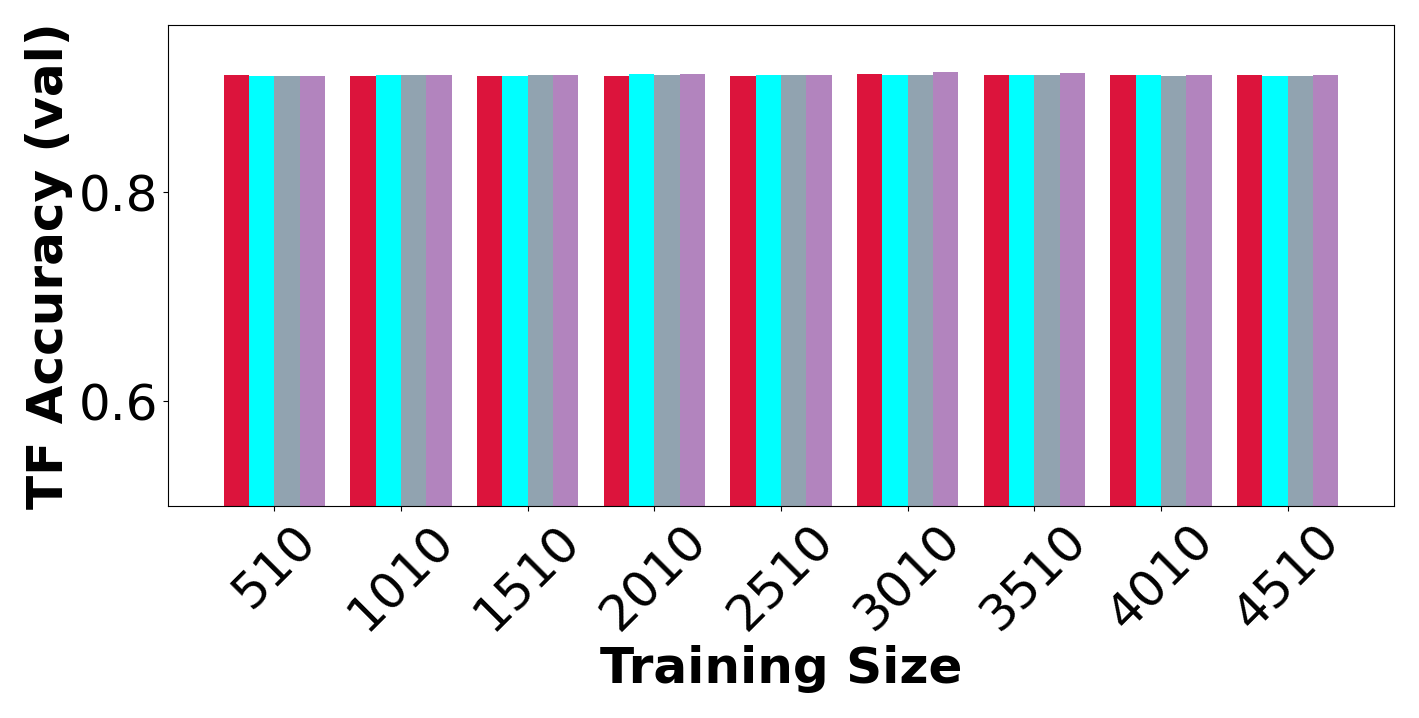}
        \caption{Numbers is prime task}
        \label{fig:task1_m}
    \end{subfigure}
    \hfill
    \begin{subfigure}[b]{0.42\textwidth}
        \centering
        \includegraphics[width=\textwidth]{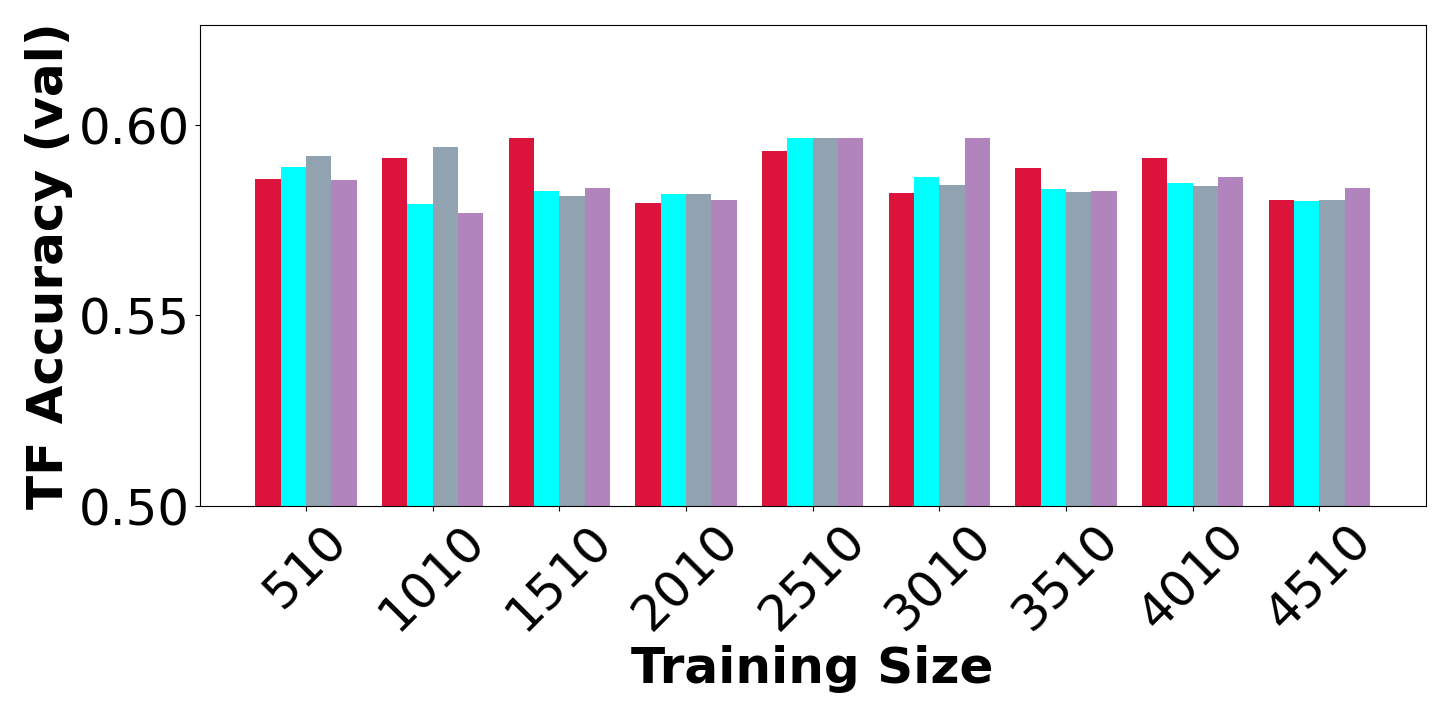}
        \caption{2D linear algebra task}
        \label{fig:task2_m}
    \end{subfigure}
    \centering
    \caption{Accuracy of LARS-VSA compared to SOTA for two mathematical reasoning tasks}
    \label{fig:total_math_2}
\end{figure}

Figures \ref{fig:task1_m} and \ref{fig:task2_m} show that LARS-VSA have better or comparable results compared to Abstractor and Transformer. 

\section{Multi-Attention Decoder}
This Decoder is used for purely and partially sequence to sequence abstract reasoning tasks. It is derived from \cite{abstractor}. 


\end{document}